\documentclass[twoside,11pt]{article}

\author{
   Matthew D. Kvalheim\footnote{\textbf{Corresponding author;} EECS Department, University of Michigan, Ann Arbor, MI, USA (\texttt{kvalheim{@}umich.edu})}
\qquad
   Brian Bittner\footnote{Robotics Institute, University of Michigan, Ann Arbor, MI, USA (\texttt{babitt{@}umich.edu})}\qquad
  Shai Revzen\footnote{Department of EECS, Department of EEB, Robotics Institute, University of Michigan, Ann Arbor, MI, USA (\texttt{shrevzen{@}umich.edu})}}

\title{Gait modeling and optimization for the perturbed Stokes regime}

\usepackage[utf8]{inputenc}
\usepackage[T1]{fontenc}
\usepackage{lmodern}
\usepackage{amsmath}
\usepackage{amssymb}
\usepackage{amsthm}

\usepackage{color}
\usepackage{transparent} 

\usepackage{stmaryrd} 
\usepackage{mathtools}
\usepackage{tikz-cd}
\usepackage{subfig}
\usepackage{lineno}
\usepackage{accents}

\newcommand{\concept}[1]{\textit{#1}}

\newcommand{\R}{\mathbb{R}}

\newcommand{\slot}{\,\cdot\,} 

\newcommand{\T}{\mathsf{T}}
\newcommand{\D}{\mathsf{D}}

\newcommand{\id}{\textnormal{id}}

\DeclarePairedDelimiter\norm{\lVert}{\rVert}

\newcommand{\Ad}{\textnormal{Ad}}
\newcommand{\ad}{\textnormal{ad}}
\newcommand{\Am}{A_{\textnormal{mech}}}
\newcommand{\Av}{A_{\textnormal{visc}}}
\newcommand{\Gm}{\Gamma_{\textnormal{mech}}}
\newcommand{\Gv}{\Gamma_{\textnormal{visc}}}
\newcommand{\Il}{\mathbb{I}_{\textnormal{loc}}}
\newcommand{\I}{\mathbb{I}}
\newcommand{\bIl}{\bar{\mathbb{I}}_{\textnormal{loc}}}
\newcommand{\Vl}{\mathbb{V}_{\textnormal{loc}}}
\newcommand{\V}{\mathbb{V}}
\newcommand{\bVl}{\bar{\mathbb{V}}_{\textnormal{loc}}}
\newcommand{\g}{\mathfrak{g}}
\newcommand{\FL}{\mathbb{F}L}
\newcommand{\Fd}{\mathbb{F}}
\newcommand{\se}{\mathfrak{se}}

\newcommand{\SE}{\mathsf{SE}}

\newcommand{\Tor}{\mathbb{T}}

\newcommand{\bo}{\mathcal{O}}

\newtheorem{Th}{Theorem}
\newtheorem{Co}{Corollary}

\newtheorem{Prop}{Proposition}
\newcommand{\bv}[1]{\accentset{\scriptstyle\circ}{#1}}

\newcommand{\thistheoremname}{}
\newtheorem*{genericthm}{\thistheoremname}
\newenvironment{thmbis}[1]%
{\renewcommand{\thistheoremname}{Theorem~\ref{#1}$'$}%
  \begin{genericthm}}
{\end{genericthm}}

\theoremstyle{definition}
\newtheorem{Def}{Definition}
\newtheorem*{Def*}{Definition}

\theoremstyle{remark}
\newtheorem{Rem}{Remark}

\newcommand{\refFig}[1]{Fig.~\ref{#1}}
\newcommand{\refSec}[1]{\S\ref{#1}}
\newcommand{\refEqn}[1]{Eqn.~\eqref{#1}}
\newcommand{\refThm}[1]{Thm.~\ref{#1}}

\usepackage{marvosym}
\usepackage[
    hypertexnames,%
    citecolor=blue,%
    colorlinks=true,%
    linkcolor=red%
]{hyperref}
\usepackage[all]{hypcap}

\usepackage{natbib}

\begin{document}

\maketitle

\begin{abstract}
Many forms of locomotion, both natural and artificial, are dominated by viscous friction in the sense that without power expenditure they quickly come to a standstill.
From geometric mechanics, it is known that for swimming at the ``Stokesian'' (viscous; zero Reynolds number) limit, the motion is governed by a reduced order ``connection'' model that describes how body shape change produces motion for the body frame with respect to the world.
In the ``perturbed Stokes regime'' where inertial forces are still dominated by viscosity, but are not negligible (low Reynolds number), we show that motion is still governed by a functional relationship between shape velocity and body velocity, but this function is no longer linear in shape change rate.
We derive this model using results from singular perturbation theory, and the theory of noncompact normally hyperbolic invariant manifolds (NHIMs).

Using the theoretical properties of this reduced-order model, we develop an algorithm that estimates an approximation to the dynamics near a cyclic body shape change (a ``gait'') directly from observational data of shape and body motion.
This extends our previous work which assumed kinematic ``connection'' models.
To compare the old and new algorithms, we analyze simulated swimmers over a range of inertia to damping ratios.
Our new class of models performs well on the Stokesian regime, and over several orders of magnitude outside it into the perturbed Stokes regime, where it gives significantly improved prediction accuracy compared to previous work.

In addition to algorithmic improvements, we thereby present a new class of models that is of independent interest.
Their application to data-driven modeling improves our ability to study the optimality of animal gaits, and our ability to use hardware-in-the-loop optimization to produce gaits for robots.
\end{abstract}

\tableofcontents

\section{Introduction}\label{sec:pertS-intro}
In this paper, we study how animals and robots move through space by deforming the ``shape'' of their body --- typically in a cyclic fashion --- to propel that body.
We call such motion-producing cyclic shape deformations \concept{gaits}.
We study a class of locomotion which includes swimming and crawling in viscous media, in which the viscous damping forces are large compared to the inertia of the body.
A classic exposition of such locomotors ``living life at low Reynolds number'' is given in \citet{purcell1977life}.
An important aspect of our work is that we consider the \concept{perturbed Stokes regime} \citep{eldering2016role} in which the inertia-damping ratio (or Reynolds number) is small but nonzero, as opposed to previous geometric mechanics literature addressing only the viscous or \concept{Stokesian limit} which formally assumes the inertia-damping ratio \emph{is} zero \citep{kelly1996geometry, kelly1995geometric, hatton2011geometric, hatton2013geometric, bittner2018geom}.
We note that our methods are related to the realization of nonholonomic constraints as a limit of friction forces \citep{brendelev1981realization, karapetian1981realizing, eldering2016realizing}.

For both scientific and engineering purposes, it is often of interest to ask whether a particular gait is optimal with respect to a goal function.
For animal locomotion, explicit equations of motion are nigh impossible to come by, and therefore directly testing animal gait optimality via analytical tools like the calculus of variations is not an option.
However, if a model can be obtained from experimental data for the local dynamics on a tubular neighborhood of the gait cycle --- i.e. a model valid for small variations in the gait cycle --- then local optimality tests can be formulated and evaluated on these models.
Such an approach was taken in \citet{bittner2018geom}, which introduced an algorithm informed by both geometric mechanics and data-driven techniques for studying oscillators \citep{RevGuk08, RevzenPhD09, revzen2015_SPIE}.

One limitation of \citet{bittner2018geom} was the assumption that motion was entirely kinematic, effectively assuming that the inertia-damping ratio is zero by assuming a \concept{viscous connection}-based model as introduced by \citet{kelly1995geometric} and to be discussed more below.
The real-world systems we are interested in have small --- but always nonzero --- inertia-damping ratio, and therefore we are interested in the extent to which the algorithm of \citet{bittner2018geom} can be improved.

By applying normally hyperbolic invariant manifold (NHIM) theory \citep{fenichel1971persistence, fenichel1974asymptotic, fenichel1977asymptotic, hirsch1977, fenichel1979geometric,eldering2013normally} in a singular perturbation context, we show that an exponentially stable invariant \concept{slow manifold} exists for small inertia-damping ratio (this was also shown in \citet{eldering2016role}).
Furthermore, this slow manifold is close to the viscous connection (viewed geometrically as a subbundle --- hence as a submanifold --- of state space), and therefore the dynamics restricted to the slow manifold are close to those assumed in the purely viscous case \citep{kelly1996geometry, kelly1995geometric, hatton2011geometric, bittner2018geom}, and reduce to those in the zero inertia-damping ratio limit.
Aside from its theoretical appeal, this result also has practical implications: it is possible to explicitly compute ``correction terms'' which, when added to the purely-viscous connection model, yield the dynamics restricted to the slow manifold.
The slow-manifold dynamics are provably more accurate than those of the idealized viscous connection model.
Additionally, they still enjoy the same useful properties of reduced dimension and symmetry under the group.
The computation of such correction terms is a fundamental technique in geometric singular perturbation theory \citep{fenichel1979geometric, jones1995geometric}, and has been used, e.g., to compute reduced-order models of robots with flexible joints \citep{spong1987integral}.

Given an algorithm that produces a data-driven local model of dynamics near a gait, we could conduct variational tests for local optimality of that gait with respect to any cost functional that the model allows us to evaluate.
Thus we have in mind two classes of application for the approach we present below: a biological application --- verification of whether a postulated goal function is optimized for an observed animal gait, and an engineering application --- optimization of robot gaits with ``hardware-in-the-loop'' by iteratively modeling and improving the gait with respect to a goal functional without the need for precise models of the robot or its interactions with the environment.

It is clear why our approach would be a boon to biology.
In most cases we cannot cajole animals to vary their gaits and observe whether that improves them.
Additionally, we rarely have detailed enough models of animal-environment interaction to allow gait optimality to be assessed from a model.

The value to gait optimization of robots comes from the fact that a gait, being a periodic continuous function of shape, is an infinite-dimensional object.
Thus, gait parameterizations are unavoidably of high dimension.
Any gradient calculation for optimization of a gait thus requires many tests to identify the influence of these many parameters.
Combined with the high practical cost of hardware experiments in terms of time and robot wear-and-tear, this renders hardware-in-the-loop optimization nigh infeasible.
We propose that by producing a tractably computable local model, we can resolve this problem.
The high-dimensional gradients can be computed by simulating the (local) model instead of directly using the hardware, decoupling the dimension of the gait parameterization from the number of experiments conducted on hardware.

It is our hope that, through a combination of geometric mechanics and NHIM theory, we can develop an algorithm which can serve the purposes of both biologists and engineers.

\subsection{Acknowledgements}
The authors were supported by NSF CMMI 1825918 and ARO grants W911NF-14-1-0573 and W911NF-17-1-0306 to Revzen.
Kvalheim would like to thank Jaap Eldering for introducing him to the relevance of NHIM theory to locomotion, for helpful comments and suggestions regarding the global asymptotic stability of the slow manifold of \refThm{th:pertS_NAIM_persists}, and for other useful suggestions.

\section{Background}

In studying locomotion, we will consider dissipative Lagrangian mechanical systems on a product configuration space $Q = S \times G$ with coordinates $(r,g)$, and with a Lagrangian of the form kinetic minus potential energy.
Here $S$ is the \concept{shape space} of the locomoting body, and $G$ is a Lie group (typically a subgroup of the Euclidean group $\SE(3)$ of rigid motions) representing the body's position and orientation in the world.\footnote{In a formal sense, one may start with generalized coordinates $Q$ and the action of $G$, and \emph{define} $S$ as a quotient manifold $Q/G$. %
    The details of this construction are not germane to our argument. %
    Instead, for simplicity we postulate the separation of configuration into ``shape'' and ``body-frame'' here, with the more general case treated in the appendices. %
}
We assume throughout this paper that $S$ is compact.
We will also assume that this system is subjected to external viscous drag forces which are linear in velocity.\footnote{We make this assumption for simplicity. %
    In principle, it should be possible to relax this assumption to derive modified but similar results for a force depending nonlinearly on velocities, as long as the linear approximation (with respect to velocities) of this force satisfies the same assumptions that we impose on our assumed linear force. %
}

If the physics of locomotion are independent of the body's position and orientation, then the Lagrangian $L(r,g,\dot{r},\dot{g})$ is independent of $g,\dot{g}$ and the viscous drag force $F_R(r,g,\dot{r},\dot{g})$ is equivariant in $g$ (on the $g,\dot{g}$ components).
Under this symmetry assumption, \citet{kelly1996geometry} derived general equations of motion satisfied by $g$ and by the \concept{body momentum}\footnote{
	Here $\g^*$ is the vector space dual of the Lie algebra $\g$ of $G$.
} $p\in \g^*$; these equations are essentially special cases of those derived in \citet{bloch1996nonholonomic}.
For a detailed statement and derivations of these equations, see \refSec{app:equations-derivation}.

Let us suppose that the kinetic energy metric of the body is scaled by a dimensionless inertial parameter $m > 0$, that the viscous drag force $F_R$ is scaled by a dimensionless damping parameter $c > 0$, and define $\epsilon \coloneqq \frac{m}{c}$ the dimensionless ratio of the two which is (up to scale) the Reynolds number in the case of fluid dynamics.
\citet{kelly1996geometry} showed that in the limit $\epsilon\to 0$, the equation of motion for $g$ becomes independent of $p$.
Defining the \concept{body velocity}\footnote{
	The body velocity is often written $g^{-1}\dot{g}$ by an abuse of notation which is only defined on matrix Lie groups where the product of a tangent vector and a group element is naturally defined. %
    For a general definition note that $\dot g \in \T_g G$, and the derivative of the left action $\D \mathrm{L}_{g^{-1}}$ restricts to a map $\T_g G \to \T_e G \cong \g$. %
    Hence the definition above.}
$\bv{g} \coloneqq \D\mathrm{L}_{g^{-1}}\dot g$, they obtained
\begin{equation}\label{eq:visc-conn-model}
\bv{g}=-\Av(r) \cdot \dot{r},
\end{equation}
where $\Av$ is called the \concept{local viscous connection}.

Away from the Stokes limit, \citet{eldering2016role} studied the \concept{perturbed Stokes regime} in which $\epsilon$ is assumed to be small but nonzero.
For $\epsilon$ sufficiently small they showed there is an exponentially stable invariant \concept{slow manifold} $M_\epsilon$, to which the dynamics converge.
We derive similar results tailored for our applications in \refSec{app:reduction-pert-stokes}.
Using an asymptotic series expansion for the slow manifold, in \refSec{app:reduction-pert-stokes} we also prove that the equations of motion for trajectories within $M_\epsilon$ take the form given by \refThm{th:pertS-dynamics-on-slow-manifold-simpler} below.
Hence trajectories of the full dynamics converge to solutions of Eqn. \eqref{eq:pertS-dynamics-slow-mfld-solved-for-BV-simpler-0} below, after a transient duration that goes to zero with $\epsilon$.

\begin{Th}\label{th:pertS-dynamics-on-slow-manifold-simpler}
	Assume that the shape space $S$ is compact.
	For sufficiently small $\epsilon > 0$, there exist smooth fields of linear maps $B(r)$ and bilinear maps $G(r)$ such that the dynamics restricted to the slow manifold $M_\epsilon$ satisfy
	\begin{equation}\label{eq:pertS-dynamics-slow-mfld-solved-for-BV-simpler-0}
	\bv{g} = -\Av(r) \cdot \dot{r} + \epsilon B(r)\cdot\ddot{r} +\epsilon G(r)\cdot(\dot{r},\dot{r}) + \bo(\epsilon^2).
	\end{equation}
\end{Th}
\begin{Rem}
	The bilinear maps or $(1,2)$ tensors $G(r)$ are \emph{not}, in general, symmetric: e.g., they are unlike Hessians.
\end{Rem}

\citet{bittner2018geom} developed a data-driven algorithm for approximating the equations of motion of a locomotion system assuming the model of \refEqn{eq:visc-conn-model}.
Here we define and study an extension of their approach to models of the form of \refEqn{eq:pertS-dynamics-slow-mfld-solved-for-BV-simpler-0}.
We examine the efficacy of this extension in modeling motion in the perturbed Stokes regime, in which $\epsilon$ is allowed to be small but nonzero.

\section{Estimating Data-Driven Models in the Perturbed Stokes Regime}
\label{sec:pertS-estimate-A}
\newcommand{\COT}{\text{COT}}
In this section, we develop a data-driven algorithm for estimating the dynamics \refEqn{eq:pertS-dynamics-slow-mfld-solved-for-BV-simpler-0} in a neighborhood of an exponentially stable periodic orbit.
We assume that the image of this periodic orbit is contained in the slow manifold $M_\epsilon$ of \refThm{th:pertS-dynamics-on-slow-manifold-simpler}, and for simplicity we assume that --- on the slow manifold --- $\ddot{r}=f(r,\dot{r})$ can be written autonomously as a function of $r$ and $\dot{r}$.
Letting $\gamma(t)$ denote the shape (or $r$) component of this periodic orbit, we refer to $\gamma$ as a \concept{gait}.

\subsection{Determination of regressors for estimation of the dynamics}\label{sec:determine-regressors}
In this section we closely follow the approach of \citet{bittner2018geom} to produce a data driven model of the dynamics from an ensemble of noisy trajectories near $\Gamma\coloneqq \text{Im } \gamma$.
We extensively use the Einstein summation convention in the regression equations below.

Let $T$ be the period of $\gamma$.
Since we assume that that the exponentially stable periodic orbit is contained in the slow manifold on which $\ddot{r}$ is of the form $\ddot{r}=f(r,\dot{r})$, it follows that there is an \concept{asymptotic phase} map $\phi\colon \T S \to [0,T)$ whose derivative along trajectories is equal to one \citep{isochrons}.
Given trajectory data $(r(t),\dot{r}(t)),~t\in[t_0,t_1]$, we assign asymptotic phase values $\phi_t \coloneqq \phi(r(t),\dot{r}(t))$  to each data point using an algorithm such as that of \citet{RevGuk08}.\footnote{In principle, any circle-valued ``phase'' function of state whose derivative along trajectories is positive could be used instead of asymptotic phase. We chose to use asymptotic phase because it is dynamically meaningful and there exist algorithms to compute it.}
After grouping data points according to their phase values, we construct Fourier series models of $\gamma,\dot{\gamma},\ddot{\gamma}$ as functions of phase.\footnote{In practice the Fourier series models of $\gamma,\dot{\gamma},\ddot{\gamma}$ might be computed from their own noisy data sets, and in this case the resulting Fourier models need not be derivatives of one another. We find that the use of matched filters is helpful in mitigating this issue; see \citet{bittner2018geom,RevzenPhD09} for more details.}

Next, we select $M$ evenly spaced values of phase, $\phi_1, \ldots, \phi_M$, to obtain values $\gamma_{m} := \gamma(\phi_m), \dot{\gamma}_m := \dot{\gamma}(\phi_m),\ddot{\gamma}_m := \ddot{\gamma}(\phi_m)$ --- the shapes, shape velocities, and shape accelerations of a system that is following the gait cycle precisely.
For each $m$ we collect from our trajectory data all triples $(r_n,\dot{r}_n,\ddot{r}_n)\coloneqq (r(t_{n}),\dot{r}(t_{n}),\ddot{r}(t_{n}))$ that are sufficiently close to $(\gamma_m,\dot{\gamma}_m,\ddot{\gamma}_m)$, i.e., such that $\|r_{n} - \gamma_m\|,\|\dot{r}_{n} - \dot{\gamma}_m\|,\|\ddot{r}_{n} - \ddot{\gamma}_m\|<\kappa$ for all\footnote{The astute experimentalist realizes that since the derivative terms contain $dt$ and $dt^2$ in their units, a certain degree of numerical conditioning can be obtained by judicious choice of units for time.
} $n$, and we also collect the corresponding $\bv{g}_n$ values.
We define the offsets $\delta_n := r_n - \gamma_m$, $\dot{\delta}_n\coloneqq \dot{r}_n-\dot{\gamma}_m$, $\ddot{\delta}_n\coloneqq \ddot{r}_n-\ddot{\gamma}_m$.
Note that the range of $n$ depends on $m$, but for notational simplicity we do not display this.

Introducing coordinates and Taylor expanding, \citet{bittner2018geom} obtained from \refEqn{eq:visc-conn-model} the following expression (no sum over $m$ or $n$):
\begin{eqnarray}\label{eq:pertS-bittner-regressors}
\bv{g}^k_n \approx - \underbrace{A^k_{m,i}\dot{\gamma}_m^i}_{C^k_{0,m}} - \underbrace{A^k_{m,i}}_{C^k_{1,m}}\dot{\delta}^i_n - \underbrace{\frac{\partial A^k_{m,i}}{\partial r^j}\dot{\gamma}_m^i}_{C^k_{2,m}}\delta^j_n -  \underbrace{\frac{\partial A^k_{m,i}}{\partial r^j}}_{C^k_{3,m}}\delta^j_n \dot{\delta}^i_n.
\end{eqnarray}
Omitted here are higher-order terms, the subscript of $\Av$, and the nonlinear $\gamma$ dependence of the local expression $A^k_i$.
They then operationalized \refEqn{eq:pertS-bittner-regressors} as a least-squares problem, written in matrix form as follows (for each $k$ and $m$; indices $k$ and $m$ elided below for clarity):
\begin{equation}\label{eq:advancedregression}
	\begin{bmatrix} \bv{g}_1 \\ \vdots \\ \bv{g}_N \end{bmatrix} =
    \begin{bmatrix}
	1,  &  \delta_1, & \dot{\delta}_1, & \delta_{1}\otimes\dot{\delta}_{1} \\
	\vdots & \vdots & \vdots & \vdots  \\
    1,  &  \delta_{N}, & {\dot\delta}_{N}, & \delta_{N}\otimes\dot{\delta}_{N} \end{bmatrix}
\cdot
    \begin{bmatrix}
	\widehat{C}_0 \\ \widehat{C}_1 \\ \widehat{C}_2 \\ \widehat{C}_3
	\end{bmatrix}
\end{equation}
where $\widehat{~}$ indicates ``estimated'' and $\otimes$ is the outer product.
For a $d$-dimensional shape space, the row of unknowns on the right consists of $1+d+d+d^2$ elements.
Once they have computed a least squares model for every $m$, they construct Fourier series so that the $\widehat{C}_i$ may be smoothly interpolated at any phase value.
The result is a local model of \refEqn{eq:visc-conn-model}.

In the perturbed Stokes regime which we seek to model, we follow a similar approach by expanding \refEqn{eq:pertS-dynamics-slow-mfld-solved-for-BV-simpler-0} instead of \refEqn{eq:visc-conn-model}.
We obtain (no sum over $m$ or $n$):
\begin{equation}
\begin{split}\label{eq:taylor-expanded-corrections}
\bv{g}^k_n &\approx - A^k_{m,i}\dot{\gamma}^i_m - A^k_{m,i}\dot{\delta}^i_n - \frac{\partial A^k_{m,i}}{\partial r^j}\delta^j_n \dot{\gamma}^i_m - \frac{\partial A^k_{m,i}}{\partial r^j}\delta^j_n \dot{\delta}^i_n + \epsilon \left( B^k_{m,i}\ddot{\gamma}^i_m + B^k_{m,i}\ddot{\delta}^i_n + \frac{\partial B^k_{m,i}}{\partial r^j}\delta^j_n \ddot{\gamma}^i_m \right. \\ \ldots & + \frac{\partial B^k_{m,i}}{\partial r^j}\delta^j_n \ddot{\delta}^i_n + G^k_{m,i,j}\dot{\gamma}^i_m \dot{\gamma}^j_m + G^k_{m,i,j}\dot{\gamma}^i_m \dot{\delta}^j_n + G^k_{m,i,j}\dot{\delta}^i_n \dot{\gamma}^j_m + G^k_{m,i,j}\dot{\delta}^i_n \dot{\delta}^j_n \\
\ldots & + \left. \frac{\partial G^k_{m,i,j}}{\partial r^{\ell}}\delta^\ell_n \dot{\gamma}^i_m \dot{\gamma}^j_m + \frac{\partial G^k_{m,i,j}}{\partial r^{\ell}}\delta^\ell_n \dot{\gamma}^i_m \dot{\delta}^j_n + \frac{\partial G^k_{m,i,j}}{\partial r^{\ell}}\delta^\ell_n \dot{\delta}^i_n \dot{\gamma}^j_m  + \frac{\partial G^k_{m,i,j}}{\partial r^{\ell}}\delta^\ell_n \dot{\delta}^i_n \dot{\delta}^j_n \right).
\end{split}
\end{equation}
	Partitioning these terms according to their dependence on the observations $\delta$, $\dot{\delta}$, and $\ddot{\delta}$, we obtained
	\begin{equation}
	\begin{split}
	\bv{g}^k_n &\approx \left(- A^k_{m,i}\dot{\gamma}^i_m + \epsilon B^k_{m,i}\ddot{\gamma}^i_m + \epsilon G^k_{m,i,j}\dot{\gamma}^i_m \dot{\gamma}^j_m\right) +
	\left(-\frac{\partial A^k_{m,j}}{\partial r^i} \dot{\gamma}^j_m + \epsilon \frac{\partial B^k_{m,j}}{\partial r^i}\ddot{\gamma}^j_m + \epsilon \frac{\partial G^k_{m,j,\ell}}{\partial r^{i}}\dot{\gamma}^j_m \dot{\gamma}^\ell_m \right)\delta^i_n \\
	\ldots & + \left(- A^k_{m,i} + \epsilon G^k_{m,j,i}\dot{\gamma}^j_m + \epsilon G^k_{m,i,j} \dot{\gamma}^j_m\right)\dot{\delta}^i_n + \left(- \frac{\partial A^k_{m,j}}{\partial r^i} + \epsilon \frac{\partial G^k_{m,\ell, j}}{\partial r^{i}}\dot{\gamma}^\ell_m + \epsilon \frac{\partial G^k_{m,j,\ell}}{\partial r^{i}} \dot{\gamma}^\ell_m\right)\delta^i_n\dot{\delta}^j_n\\
	\ldots & + \epsilon \left( B^k_{m,i}\,\ddot{\delta}^i_n + \frac{\partial B^k_{m,j}}{\partial r^i} \,\delta^i_n\ddot{\delta}^j_n+ G^k_{m,i,j}\,\dot{\delta}^i_n\dot{\delta}^j_n +  \frac{\partial G^k_{m,j,\ell}}{\partial r^{i}} \,\delta^i_n \dot{\delta}^j_n \dot{\delta}^\ell_n\right),
	\end{split}
	\end{equation}
	giving a similar least squares problem written in matrix form as follows (for each $k$ and $m$; indices $k$ and $m$ elided below for clarity):
\begin{equation}\label{eq:advancedregression2}
	\begin{bmatrix} \bv{g}_1 \\ \vdots \\ \bv{g}_N \end{bmatrix} =
    \begin{bmatrix}
	1,  &  \delta_1, & \dot{\delta}_1, & \ddot{\delta}_1 & \delta_{1}\otimes\dot{\delta}_{1} & \delta_{1}\otimes\ddot{\delta}_{1} & \dot{\delta}_{1}\otimes\dot{\delta}_{1} & \delta_{1}\otimes\dot{\delta}_{1}\otimes\dot{\delta}_{1} \\
	\vdots & \vdots & \vdots & \vdots & \vdots & \vdots & \vdots & \vdots \\
    1,  &  \delta_{N}, & {\dot\delta}_{N}, & \ddot{\delta}_N & \delta_{N}\otimes\dot{\delta}_{N} & \delta_{N}\otimes\ddot{\delta}_{N} & \dot{\delta}_{N}\otimes\dot{\delta}_{N} & \delta_{N}\otimes\dot{\delta}_{N} \otimes\dot{\delta}_{N}  \end{bmatrix}
\cdot
    \begin{bmatrix}
	\widehat{C}_0 \\ \widehat{C}_1 \\ \widehat{C}_2 \\ \widehat{C}_3 \\ \widehat{C}_4 \\ \widehat{C}_5 \\ \widehat{C}_6 \\ \widehat{C}_7
	\end{bmatrix}
\end{equation}
For a $d$-dimensional shape space, the row of unknowns on the right consists of $1+d+d+d+d^2+d^2+d^2+d^3$ elements.
Once we have computed a least squares model for every $m$, we similarly construct Fourier series so that the $\widehat{C}_i$ may be smoothly interpolated at any phase value.
The result is a local model of \refEqn{eq:pertS-dynamics-slow-mfld-solved-for-BV-simpler-0}.

Because it is the only term of order $\kappa^3$, we find that in practice the 3-index regressor $\delta \otimes \dot{\delta} \otimes \dot{\delta}$ can often be omitted if $\kappa>0$ is sufficiently small.
In the remainder of this paper, we refer to the regressors of \refEqn{eq:advancedregression2} (with the 3-index term excluded) as the ``perturbed Stokes regressors'', and refer to those used in the \citet{bittner2018geom} algorithm as the ``Stokes regressors.''

\begin{Rem}
	All tensors appearing in \refEqn{eq:pertS-bittner-regressors} and \refEqn{eq:taylor-expanded-corrections} are not necessarily symmetric, and therefore the order of terms matters.
\end{Rem}

\begin{Rem}
	Examining \refEqn{eq:pertS-bittner-regressors}, we see that there are some constraints that the regression does not enforce.
	Namely, $C_0 = \left[C_1\right]_i \dot \gamma^i$ and $C_2 = \left[C_3\right]_i \dot \gamma^i$.
	When we performed regressions ignoring these implicit constraints, we found that the constraints are not respected in the results.
	However, an important consequence of \refEqn{eq:taylor-expanded-corrections} is that, for systems operating in the perturbed Stokes regime, such a mismatch is actually to be expected --- this is because some independent new terms appear in $C_1,\ldots,C_3$ which break the constraints.
\end{Rem}

\subsection{Local models enable optimality testing and optimization}\label{sec:utility-of-models}
The data-driven models computed by the process described above have predictive power locally, in a neighborhood of a gait cycle.
For any shape trajectory inside this neighborhood, we can used the local model to predict the trajectory of the body in the world.
We assume that we are interested in some $\R$-valued goal functional $\tilde{\phi}(\gamma,g_\gamma)$ defined on an appropriate space of trajectories.
Here the group trajectory $g_\gamma(t)$ is determined by the gait $\gamma(t)$ via \refEqn{eq:pertS-dynamics-slow-mfld-solved-for-BV-simpler-0}, and therefore we may consider the goal functional $\phi(\gamma) := \tilde{\phi}(\gamma,g_\gamma)$ to be a function of $\gamma$ alone.

\paragraph*{Testing for Optimality}---~
We can test the gait of an organism for optimality by checking that $0 = \frac{\partial}{\partial s}\phi(\gamma_s)|_{s=0}$
for all smooth variations $\gamma_s$ of a gait $\gamma$ (where $\gamma_0 = \gamma$).
This condition is necessary for local optimality, but depending on the choice of $\phi$ it is often possible to argue on physical grounds that its satisfaction is also sufficient for optimality.
While this variational condition can be used to derive a PDE via the Euler-Lagrange approach, a more computationally straightforward approach is to consider a finite- (but often high-) dimensional family $\gamma_p$ with $p \in \R^N$, and numerically computing the gradient $\nabla_p \phi(\gamma_p)$.
When this gradient is sufficiently small at some parameter $p_*$, then it might be possible to argue that the gait is nearly extremal (or possibly optimal) with respect to $\phi$.\footnote{In some cases this procedure is provably correct.
    Furthermore, suitable finite-dimensional families that provide these guarantees always exist \citep[Sec.~16]{milnor1969morse}.
    We do not discuss these technicalities any further here.}
    Since we can compute $\phi$ using a data-driven model around $\gamma_p$, we can compute $\nabla_p \phi(\gamma_p)$.
    We can do so \emph{directly from observation} and without need for any general model of body-environment interactions, so long as use of \refThm{th:pertS-dynamics-on-slow-manifold-simpler} can be justified.

\paragraph*{Optimizing Gaits}---~
We can use the gradient $\nabla_p \phi(\gamma_p)$ to iteratively improve the gait of a robot whose dynamics satisfy \refThm{th:pertS-dynamics-on-slow-manifold-simpler} without requiring any further details of the physics.
Taking parameter set $p$ we compute the next iterate $p' := p + \alpha \nabla_p \phi(\gamma_p)$, with the step-size scaling $\alpha>0$ chosen to ensure that $p'$ is within the domain for which our local model of $\phi$ is valid, using the approach of \citet[Sec. 7.2]{bittner2018geom}.
For each gait $\gamma_p$, we only require enough experimental data for building a good local model of $\phi$ near $\gamma_p$ --- a dataset whose size does not depend on the dimension of the representation $p$.
We plan to use this decoupling to perform hardware-in-the-loop optimization to produce rapid adaptation of robot motions in the face of foreign environments, mechanical failures, and more.

\section{Performance Comparison of the Two Data-Driven Models}\label{sec:performance}

One of the primary contributions of this paper is the introduction of new regressors based on \refThm{th:pertS-dynamics-on-slow-manifold-simpler}, which we use to augment the regressors used in the algorithm of \citet{bittner2018geom} for estimating the dynamics near a gait.
These allow us to extend the domain of validity of their algorithm from the Stokesian limit to include the perturbed Stokes regime.
To demonstrate this, we constructed a swimming model which we simulated at various Reynolds numbers, and tested the ability of the two types of local models to predict the results of the fully nonlinear simulation.\footnote{All of these simulations did not account for fluid-fluid interactions; as such we make no claim that they are physically meaningful at the higher Reynolds number in the ranges shown.
}
\subsection{Modeling a swimmer}

\begin{figure}
	\centering

		\def\svgwidth{\textwidth}
		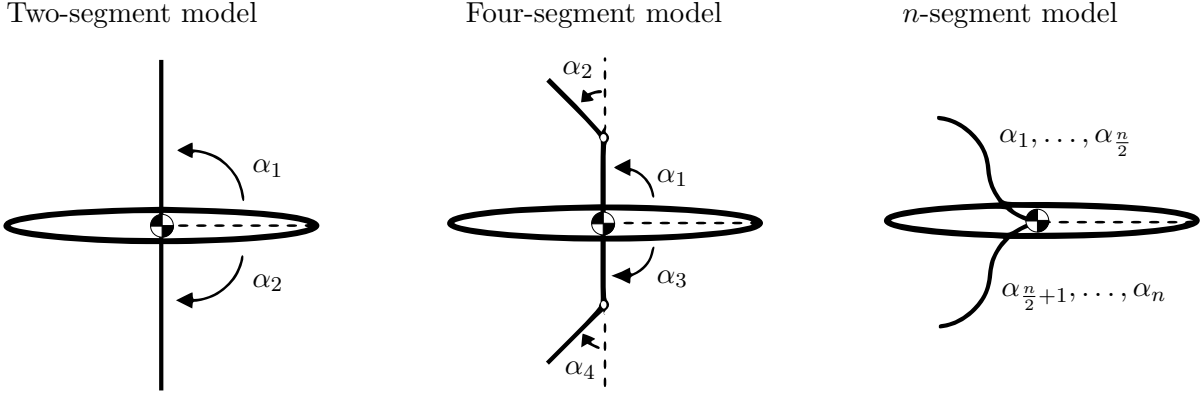
		\caption{Schematic representation of our swimming model. %
        A single body (ellipses with center of mass marked) of mass $m$ and moment of inertia $m\bar{I}$ is attached to two identical paddles each comprising 1 (left), 2 (middle), or $\frac{n}{2}$ (right) segments. %
        The length of the body is $L$, and the total length of each of the two paddles is $d$. The length of each segment is $\frac{d}{n}$. %
        } \label{fig:system}
\end{figure}

We tested the prediction quality of both models on a  swimming model.
The system shown in \refFig{fig:system} had uniformly distributed mass along a central body, with two paddles comprising chains of massless links extending from the center of the body.
Each paddle could be broken up into an arbitrary number $\frac{n}{2}$ ($n$ even) of equally spaced links, which sum to a constant total length independent of $n$.
This allowed us to vary the behavior of the system from one reminiscent of a boat with oars (for $n = 2$) to one more like a bacterial cell with flagella (for $n$ large).

The system moves in a homogeneous and isotropic plane.
Its configuration space is $S\times G = \Tor^{n} \times \SE(2)$: the $n$-torus and the special Euclidean group of planar rigid motions $\SE(2)$.
We assume the dynamics are equivariant under $\SE(2)$.
The group element $g \in \SE(2)$ provides the position and orientation of the central body in world coordinates with respect to a fixed inertial reference frame.
Hereon we represent $g$ as a column vector $g = [x,y,\theta]^T$, and similarly represent $\dot{g}$ as a column vector.
We define the body velocity
\begin{equation}
\bv{g} = \begin{bmatrix}  \cos(\theta) & \sin(\theta) & 0 \\
-\sin(\theta) & \cos(\theta) &
0 \\ 0 & 0 & 1 \end{bmatrix} \dot{g}.
\end{equation}
We treat the link at the main body (length $L$) and the links comprising the paddles (length $d$) as slender members, and model their drag forces according to Cox theory \citep{cox1970motion} using the drag matrices
\begin{equation}
C_{\frac{d}{n}} =
c \begin{bmatrix}
C_x \frac{d}{n} & 0 & 0\\ 0 & C_y \frac{d}{n}& 0\\0 & 0 & \frac{1}{12} (\frac{d}{n})^3 C_y
\end{bmatrix}, \quad
C_{_L} =
c \begin{bmatrix}
C_x L & 0 & 0\\ 0 & C_y L& 0\\0 & 0 & \frac{1}{12}L^3 C_y
\end{bmatrix},
\end{equation}
where the factor $c > 0$ is explicitly written for later scaling purposes.
The drag coefficient ratio $C_y/C_x$ has a maximum value of $2$ corresponding to the limit of infinitesimally thin segments, and we will assume this limiting ratio here (c.f. \citet[Sec.~2.B]{hatton2013geometric}).
Given these drag matrices, the wrench on the central link can be written as
\begin{equation}
F_{\textnormal{body}} = c \bar{F}_{\textnormal{body}} =   -C_{_L}
\bv{g}.
\end{equation}
The wrench that the segments (denoted $i$) apply on the body can be written as
\begin{equation}
F_{i} = c \bar{F}_i
= -W_i
C_{\frac{d}{n}}
V_i
\begin{bmatrix}
\bv{g} \\ \dot{\alpha}
\end{bmatrix},
\end{equation}
where the linear map $W_i(g,\alpha)\colon \se(2)^*\to \se(2)^*$ maps a wrench on link $i$ to a wrench on the body and the linear map $V_i(g,\alpha)\colon \se(2)\to\se(2)$ maps a velocity in the body frame to a velocity in the link frame. Let $R_\beta$ denote the counterclockwise rotation of the plane by angle $\beta$, define $e_2\coloneqq [0,1]^T$, and write $\bv{g} = [\bv{g}_{x,y}^T,\dot{\theta}]^T$.
Then, for the $n$-segment model (recall that $n$ must be even), for $i \in \{1,\ldots, n\}$ the linear maps $V_i$ and $W_i$ are given by
\begin{equation}
\begin{split}
V_i \cdot \begin{bmatrix}\bv{g}\\\dot{\alpha}\end{bmatrix} &=  \begin{bmatrix}R_{\alpha_*+\cdots+\alpha_i}^{-1}\bv{g}_{x,y}+ \left(\frac{d}{2n}\left(\dot{\theta}+\sum_{k=*}^i\dot{\alpha}_k\right)+\frac{d}{n}\sum_{k=*}^{i-1}\left(\dot{\theta}+\sum_{j=*}^k\dot{\alpha}_j\right)R^{-1}_{\alpha_{k+1}+\cdots+\alpha_i}\right) e_2 \\ \dot{\theta} + \sum_{k=*}^i\dot{\alpha}_k\end{bmatrix}\\
W_i \cdot \begin{bmatrix}f\\\tau\end{bmatrix} &= \begin{bmatrix}R_{\alpha_*+\cdots+\alpha_i}f\\ \tau + e_2^T\left(\frac{d}{2n}I_{2\times 2} + \frac{d}{n}\sum_{k=*+1}^{i}R_{\alpha_k+\alpha_{k+1}+\cdots + \alpha_i} \right)\cdot f\end{bmatrix},
\end{split}
\end{equation}
where $* \coloneqq 1+ \llfloor i / \frac{n}{2}\rrfloor \cdot \frac{n}{2}\in \{1,\frac{n}{2}+1\}$, $f = [f_1,f_2]^T$, and where a summation is understood to be zero if the lower bound of its index set exceeds its upper bound.

These wrenches act on the body (which has uniformly distributed mass $m$ and moment of inertia $I = m\bar{I}$ about its midpoint) yielding the following equations of motion in world coordinates:
\begin{equation} \label{eq:pbdyn}
\ddot{g} = \begin{bmatrix} \ddot{x} \\ \ddot{y} \\ \ddot{\theta}  \end{bmatrix} =
\frac{1}{\epsilon} \begin{bmatrix}
1 & 0 & 0 \\ 0 & 1 & 0 \\ 0 & 0 & \frac{1}{\bar{I}}
\end{bmatrix}
\begin{bmatrix}  \cos(\theta) & -\sin(\theta) & 0 \\
\sin(\theta) & \cos(\theta) &
0 \\ 0 & 0 & 1 \end{bmatrix}
\Bigg(
\bar{F}_{\textnormal{body}} + \sum_{i=1}^n\bar{F}_{i}
\Bigg),
\end{equation}
where $\epsilon \coloneqq \frac{m}{c}$ is the dimensionless inertia-damping ratio.
In keeping with our earlier conventions that $m$, $c$, and $\epsilon$ are all dimensionless
we think of the ``$1$'' terms on the diagonal in \refEqn{eq:pbdyn} as having units of inverse time.

Upon inspection of \refEqn{eq:pbdyn}, we see that by modifying $\epsilon$ we can directly adjust the ratio of inertial to viscous forces in the swimming model.
The Stokesian limit corresponds to $\epsilon\to 0$; on the other hand, the $\epsilon \to \infty$ limit corresponds to a fully ``momentum-dominated'' regime, wherein viscous effects are negligible and motion is governed by conservation of momentum via Noether's theorem (see Corollary \ref{co:pertS-Noether} \refSec{sec:pertS-mech-visc-conn}).
In the following \refSec{sec:model-accuracy-testing} we simulate the swimming model at a variety of $\epsilon$ values, and compare the performance of the two algorithms for estimating the dynamics near a gait cycle.

\subsection{Comparison of the estimated models}\label{sec:model-accuracy-testing}
In all simulations in this section, we used the parameter values $L = 1$, $d = 0.5$, $C_x = 1$, $C_y = 2$, and $\bar{I} = 1$.
The only remaining free variable is $\epsilon$, which governs both the ratio of inertial to viscous forces and the rate of attraction to the slow manifold.
The procedure we used for generating simulations for experiments in this section is identical to that described in \citet{bittner2018geom}.
Briefly, an experiment consists of 30 cycles of a numerically integrated stochastic differential equation (SDE) representing shape space dynamics consisting of a deterministic oscillator perturbed by system noise (see \citet[Sec.~6.2]{bittner2018geom} for precise details on the SDE, parameter values used, etc.).

We used these noisy shape dynamics to drive the body momentum and group dynamics via the full equations of motion \refEqn{eq:pertS-equations-motion-local-w-epsilon} derived in \refSec{sec:pertS-reduction-Stokes-limit}.
For each simulation we recorded a ``ground truth'' body velocity trajectory $\bv{g}_G$.
We used this record to evaluate the accuracy of the data-driven approximations.
We denoted the body velocity computed with the perturbed Stokes regressors by $\bv{g}_{p}$, and those computed with the Stokes regressors by $\bv{g}_{s}$.

As a ``zeroth-order'' phase model of the dynamics, we constructed a Fourier series model of $\bv{g}_G$ with respect to the estimated phase (see \refSec{sec:determine-regressors}), which we denote by $\bv{g}_a$.
For any data point, the zeroth-order model prediction is $\bv{g}_a(\varphi)$ for the phase $\varphi$ of that data point.

We computed the RMS errors $e^k_*$ for each component $k$ of the body velocity and each model $* = p,s,a$ by $e^k_* := \langle|\bv{g}^k_*-\bv{g}^k_G|^2\rangle^{1/2}$.
Since the numerical value of these errors means little, we defined the metric $\Gamma^k_* := 1 - e^k_* / e^k_a$ for $* = p,s$ to indicate how much better the regression models were performing compared to the zeroth-order phase model $\bv{g}_a$.
A $\Gamma^k_*$ of $0$ indicates doing no better than the zeroth order model whereas a $1$ indicates a perfect model.
To further highlight the \emph{difference} in prediction quality, we also plot $\Delta^k := \Gamma^k_p-\Gamma^k_s$.

\begin{figure}
\centering
\includegraphics[clip, trim=0.cm 0cm 0cm 0cm, width=1.0\textwidth]{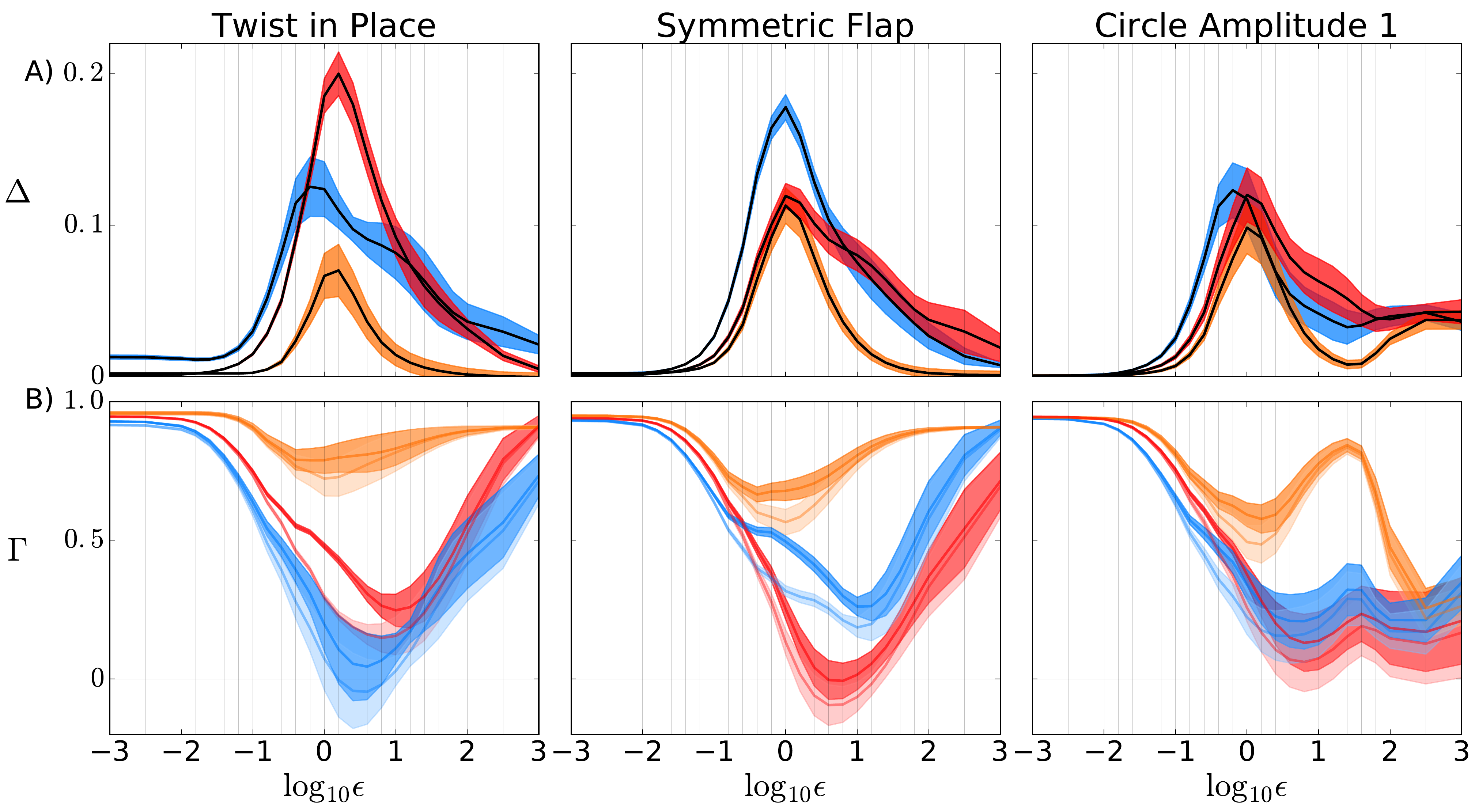}
\includegraphics[clip, trim=3.5cm 12cm 5.75cm 12cm, width=1.0\textwidth]{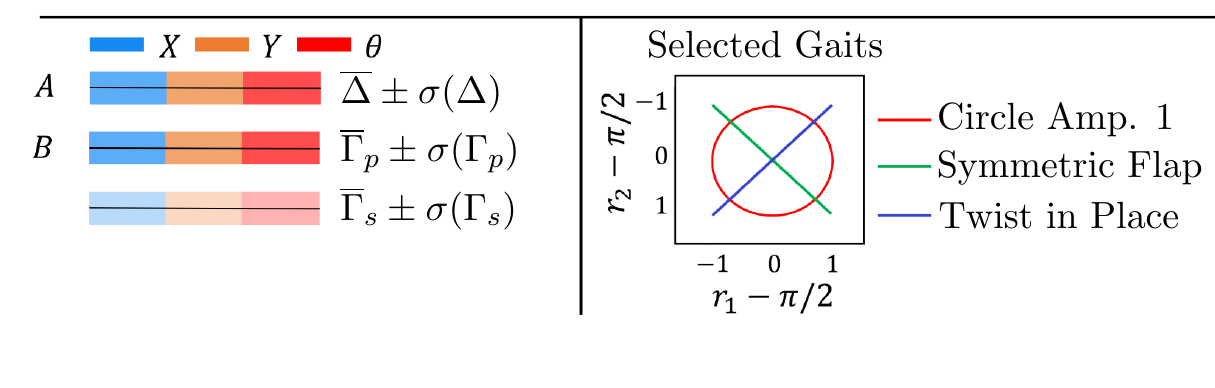}
\caption{ %
Comparison of model prediction quality when using the perturbed Stokes regressors versus the Stokes regressors on three gaits, in terms of the $\Gamma$ and $\Delta$ quality metrics. %
We have plotted the components of $\Delta$, representing the relative advantage of perturbed Stokes regressors (top row; (A)), and $\Gamma$, representing model prediction quality (bottom row; (B)), against 6 orders of magnitude variation in the inertial to viscosity ratio $\epsilon$ (logarithmic scale; sampled at 25 values (vertical gray lines). %
We present three gaits, whose shape space loci are in-phase paddle angle (which leads to anti-phase paddle motions; ``Twist in Place''; left column; blue line in shape-space plot), anti-phase paddle angle (bilaterally symmetric paddle motions; ``Symmetric Flap''; middle column; green line in shape-space plot), and quarter-cycle out of phase paddle angles (``Circle Amp. 1''; right column; red line in shape-space plot). %
All three gaits have paddle angles ranging between $-1$ and $1$ radians. %
For each value of $\epsilon$ we performed 8 simulation trials each consisting of 30 (noisy) gait cycles, and plotted mean and standard deviation of $\Delta$ and $\Gamma$ for each component of the $\mathfrak{se}(2)$ body motion ($X$ blue; $Y$ orange; $\theta$ red; saturated for $\Delta$ and $\Gamma_p$, pale for $\Gamma_s$). %
Consistently for all components and gaits, the perturbed Stokes regressors provide a better model for an order of magnitude or wider range of $\epsilon$ around $\epsilon=1$. %
For Twist in Place and Symmetric Flap gaits, both models are accurate for large and small $\epsilon$ ($\Gamma$ close to 1); for the Circle Amplitude 1 gait, the prediction is only accurate for the Stokes regime (small $\epsilon$). }
\label{fig:r1}
\end{figure}

\begin{figure}
\begin{center}
\includegraphics[clip, trim=0.cm 0cm 0cm 0cm, width=1.0\textwidth]{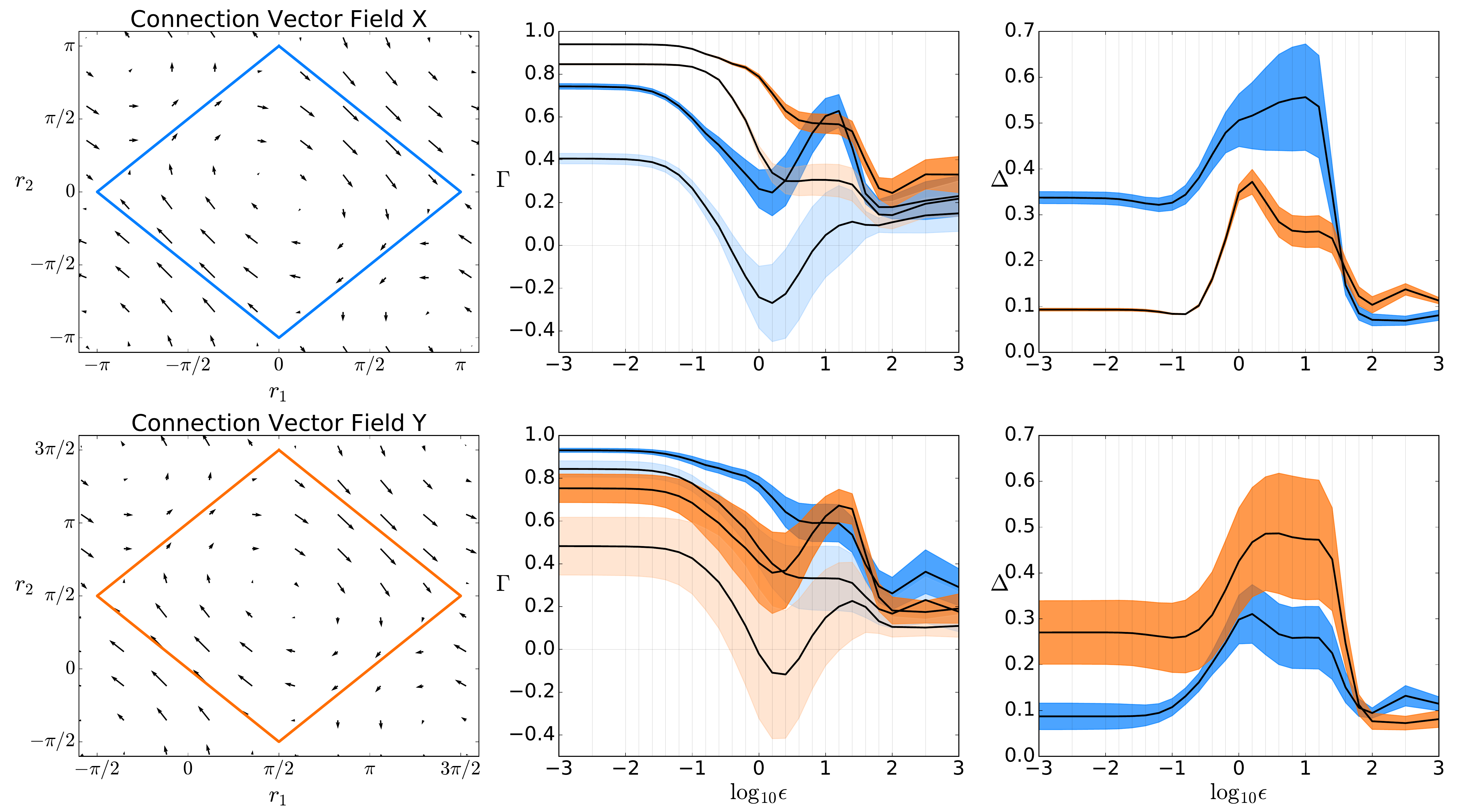}
\includegraphics[clip, trim=2.2cm 13.35cm 5.25cm 13.35cm, width=1.0\textwidth]{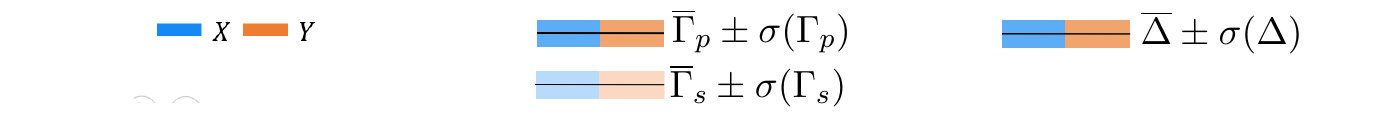}
\caption{ %
Comparison of model prediction quality when using the perturbed Stokes regressors versus the Stokes regressors on two extremal gaits, in terms of the $\Gamma$ and $\Delta$ quality metrics. %
Plots consist of the same types as those in \refFig{fig:r1}. %
We only plot the $X$ (blue) and $Y$ (orange) components of $\Gamma$ (middle column; saturated color $\Gamma_p$; pale colors $\Gamma_s$) and $\Delta$ (right column). %
We selected the gait to maximize either the $X$ component of total body frame motion (top row) or the $Y$ component (bottom row). %
The gaits are extremal in the Stokes regime ($\epsilon=0$) and selected by taking the zero level set of the connection curvature (method from \citet{hatton2011geometric, hatton2013geometric}). %
Following their approach, we plot the connection of the coordinate being optimized as a vector field over the shape-space (black arrows; left column), with the shape-space gait locus plotted over it (diamond shapes in left column, colored by coordinate optimized). %
Results show that both models are most accurate for small $\epsilon$ (the Stokes regime; $\Gamma$ closer to $1$), with the perturbed Stokes regressors providing improvements across the entire range. %
Over the two order of magnitude range of $10^{-0.5}<\epsilon<10^{1.5}$ this advantage is noticeably more pronounced (the perturbed Stokes regime; bump in $\Delta$ plots). %
Also note that the $X$ extremal gait shows much greater $\Delta^x$; the $Y$ extremal gait shows much greater $\Delta^y$.
}\label{fig:r2}
\end{center}
\end{figure}

\begin{figure}
\begin{center}
\includegraphics[clip, trim=0.cm 0cm 0cm 0cm, width=\textwidth]{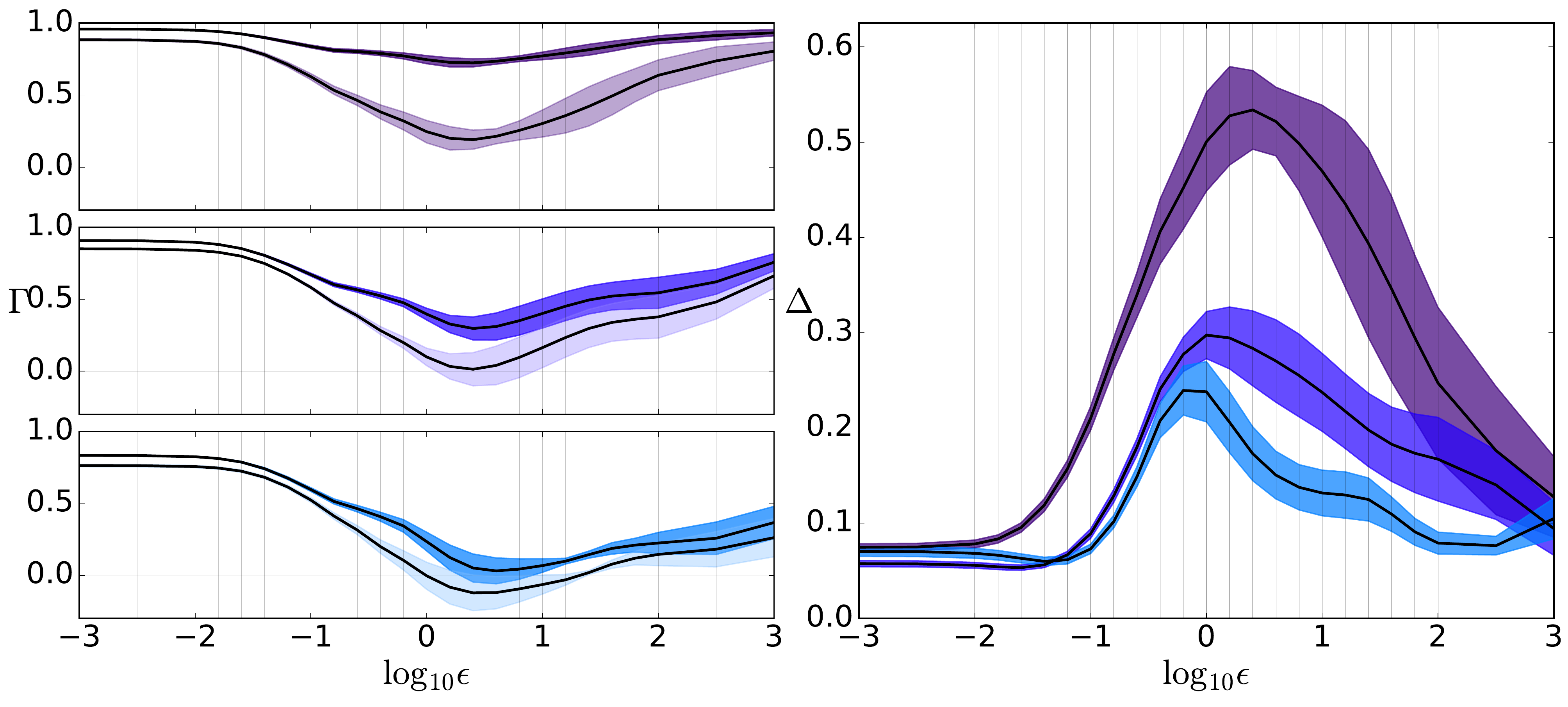}
\includegraphics[clip, trim=0cm 14.5cm 4.25cm 0cm, width=1.0\textwidth]{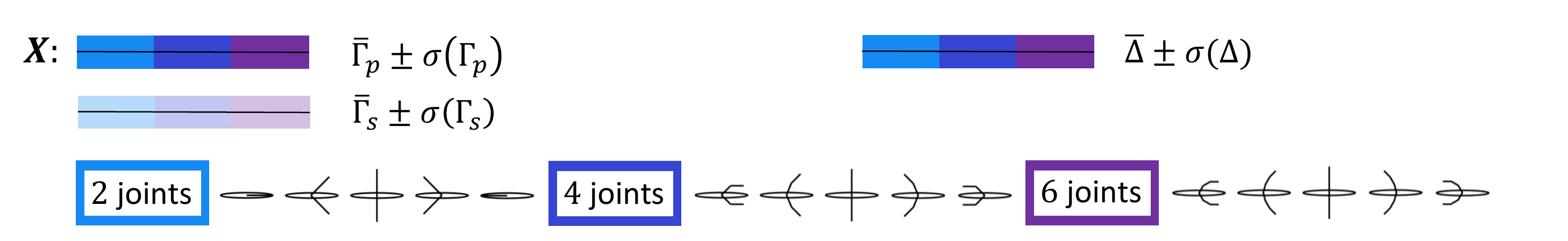}\caption{Comparison of model prediction quality when using the perturbed Stokes regressors versus the Stokes regressors on paddles with different dimensions of the shape space, shown in terms of the $\Gamma$ and $\Delta$ quality metrics. %
Plots consist of the same types as those in \refFig{fig:r1}. %
We plotted $\Gamma$ and $\Delta$ of three swimmers with different numbers of paddle segments: one segment per paddle (light blue), two segments (blue), and three segments (purple); see \refFig{fig:system} for schematic. %
We used a symmetric flapping gait (see \refFig{fig:r1}; small cartoons above). %
The paddles moved symmetrically with total angles of all joints summing up to a sinusoid of amplitude $\pi$.
We plot the $X$ components of $\Gamma$ (left column; one plot per model; saturated colors $\Gamma_p$; pale colors $\Gamma_s$) and $\Delta$ (right column). %
Results show that over the two order of magnitude range of $10^{-0.5}<\epsilon<10^{1.5}$, the perturbed Stokes regressors consistently provide improvements.
The relative improvement $\Delta$ increased markedly with shape space dimension, by as much as $0.5$ in $\Delta$.
}
\label{fig:r3}
\end{center}
\end{figure}

\subsubsection{Algorithm comparison using manually selected gaits}\label{sec:performance-manual-gaits}

We chose to first test the modeling approaches on a collection of simple manually selected behaviors.
These include behaviors we term  ``twist in place'' and ``symmetric flapping'' gaits, both of which initialize with paddles aligned at a quarter turn away from the body (as depicted in the two-segment model in Figure \ref{fig:system}), and respectively involve anti-symmetric and symmetric sinusoidal movement of the paddles with amplitude $1$.
The ``symmetric flapping gait'' primarily moves in the direction of the $x$ body axis, while the ``twist in place gait'' primarily changes the $\theta$ body coordinate.
Finally, we considered a ``circle'' gait which also initializes the paddles at a quarter turn away from the body and moves them sinusoidally with amplitude $1$, but has a quarter cycle phase offset between them.
This gait tends to move the system in a way that changes all three body coordinates throughout its execution.

We selected these three gaits because they are simple to describe and span a range of resultant body motions.
For single link paddles, the body shape space is 2D, and these gaits are represented by loci that are diagonal lines with slopes $1$, $-1$, and a circle (see \refFig{fig:r1}).
We simulated the gaits and plotted mean and variance of $\Gamma_s$, $\Gamma_p$ and $\Delta$ for each value of $\epsilon$ (\refFig{fig:r1}).
The plot shows that for all three gaits tested and for all three body coordinates, over a range spanning an order of magnitude or more around $\epsilon=1$, the perturbed Stokes models are better by $\Delta>0.05$ or more.

\subsubsection{Algorithm comparison using extremal gaits}
\label{sec:algo-compare-extremal-gaits}
Arbitrarily selected gaits such as those examined in the previous section are not expected to exhibit any special properties with respect to our modeling approach.
In particular, with respect to a goal function $\phi(\cdot)$, they are expected to be regular points of $\phi(\cdot)$. 
However, $\phi$-optimal gaits have $\nabla_p\phi = 0$ and thus have additional structure that might interact with the modeling approach.

We chose goal functionals $\int \bv{g}^x(t)\,\mathrm{d}t$ and $\int \bv{g}^y(t)\,\mathrm{d}t$ (where superscripts denote components) corresponding to displacement in the $x$ and $y$ coordinates as measured in the body frame of the paddleboat.
This is \emph{not} the same as actual $x$ or $y$ displacement in the world, since boat orientation changes over time.
Using the methods of \citet{hatton2013geometric}, we determined the extremal gaits for these goal functionals in the Stokes regime with high accuracy.
Plotted in the shape-space (and superimposed on the ``connection vector fields'' \citep{hatton2011geometric,hatton2013geometric} of the appropriate goal functional) they are diamond shaped (\refFig{fig:r2}).
We also plotted $\Gamma$ and $\Delta$, revealing that again, perturbed Stokes regressors improve performance ($\Delta>0.15$) over a range of two orders of magnitude in $\epsilon$.
Unlike the arbitrary gaits of the previous section, the extremal gaits have $\Gamma>0.1$ for all $\epsilon>1$ for both model types.
This suggests that even outside the perturbed Stokes regime the addition of regressors improves upon the zeroth order phase model.
It is also notable that in the extremal $x$ gait, $\Delta^x$ is significantly better than $\Delta^y$, whereas in the extremal $y$ gait the converse is true.

\subsubsection{Performance gains grow with shape space dimension}
\label{sec:complexity}

Thus far we have only presented results for systems having 2D shape spaces.
Because data-driven methods are often handicapped by their inability to scale with model dimensionality, we chose also to test our approach on systems of higher dimension by extending each paddle into a multi-segmented model.
We selected a gait similar to that of the symmetric flapping gait, but with the additional feature that the bending angle of a paddle was uniformly distributed through the joints it contains.
In particular, the relative angles between adjacent segments were equal and of amplitude $\pi/N$, where $N$ is the number of joints.

We plotted $\Gamma^x_p$, $\Gamma^x_s$ and $\Delta^x$ for paddles with $1$, $2$ and $3$ segments (\refFig{fig:r3}).
The $\Delta^x$ shows a marked improvement in the $4$D and $6$D models, suggesting that as shape-space complexity increased, the advantage of perturbed Stokes regressors became comparatively more significant.

\subsection{Discussion}

The results of \refSec{sec:model-accuracy-testing} show that for all versions of the swimming model and all gaits that we tested there exists a sizable window of $\epsilon$ values wherein the perturbed Stokes regressors provide models of superior quality when compared to the Stokes regressors.
In particular, the improvement is consistently present in the region $\log_{10} \epsilon \in [0,1]$, suggesting that this range of $\epsilon$ might be the range for which the predicted slow manifold is both present and sufficiently simple to be captured by the new regressors.

As noted in \refSec{sec:algo-compare-extremal-gaits}, the perturbed Stokes regressors seem to improve prediction performance more in the direction in which the gait was extremal.
We hypothesize that this is because extremal gaits have already exhausted any first-order improvements available, i.e. gradients are zero.
With the first-order terms close to zero, the presence of more high-order terms among the perturbed Stokes regressors may have a greater effect on the relative prediction error.

It is interesting to note the large magnitude of improvement in $\Delta$ as the shape space dimension increased in \refFig{fig:r3}.
Whether this is an artifact of the particular model and/or gait, or a more general feature, remains to be determined.

At the lower end $\epsilon$ magnitudes studied here, the systems are near the Stokesian limit, and therefore we expect relatively little improvement from adding regressors designed for the perturbed Stokes regime.
This is consistent with our experimental results in all figures which show for $\epsilon$ small both small values of $\Delta$ and large values of $\Gamma$ for both sets of regressors.

For very large values of $\epsilon$, the predictive quality of both algorithms is hindered by at least three factors, although only the first two can be observed here.
\begin{enumerate}
	\item The $\bo(\epsilon^2)$ term in \refThm{th:pertS-dynamics-on-slow-manifold-simpler} becomes more significant as $\epsilon$ increases.
	This issue is insurmountable if we restrict ourselves to Stokes regressors.
    If we do not, it is possible to compute correction terms which are higher order in $\epsilon$ and which can inform the selection of additional regressors for addition to our algorithm.
	It is one possible direction for future work.

	\item For $\epsilon$ sufficiently large, we expect a bifurcation in which the slow manifold (whose existence is guaranteed by \refThm{th:pertS_NAIM_persists} in \refSec{sec:pertS-reduction-perturbed-Stokes-regime}) ceases to exist.
	For such values of $\epsilon$, the hypotheses of \refThm{th:pertS-dynamics-on-slow-manifold-simpler} are not satisfied, and a reduced-order model may not exist.
	This is a mathematical expression of the physical reality of inertial effects playing a dominant role as $\epsilon$ increases, and eventually requiring momentum states to be added to the models.

	\item For sufficiently large values of $\epsilon$ the full complications of fluid-fluid interactions to come into play, and the linear viscous friction model we used becomes less and less accurate.
	We conjecture that for many systems this effect will not have significant influence until after $\epsilon$ is already sufficiently large for the slow manifold to have disappeared.
    It would be interesting to explore this issue further.
\end{enumerate}

\section{Conclusion}
We have shown that the accuracy of data-driven models motivated from geometric mechanics can be improved by using a collection of regressors derived from an asymptotic series approximation of an attracting invariant manifold in the small parameter $\epsilon$ representing the ratio of inertial to viscous forces (a Reynolds-number-like parameter).
The existence of such an invariant manifold was previously known in similar situations,\footnote{
	But see the discussion preceding  \refThm{th:pertS_NAIM_persists} in \refSec{sec:pertS-reduction-perturbed-Stokes-regime}, which details how our result differs from that of \citet{eldering2016role}.
 } as were the approximation techniques we employed, but the combination of these together for producing data-driven models of locomotion is a novel contribution.
In simulations where we tested geometrically similar motions over $6$ orders of magnitude of $\epsilon$, we obtained improvements of $5$--$65\%$ (depending on the specific system and gait) compared to previous work, suggesting that these better-informed models can indeed capture the perturbed Stokes regime more accurately.
Furthermore, the results of one of our experiments showed further improvements as the shape-space dimension of the locomoting system increased; this suggests that higher-dimensional systems might be modeled effectively using our approach.

Future work will include application of our algorithm to questions of locomotion optimality in animals, and to hardware-in-the-loop optimization of robot motions.
An additional direction for future work is the selection of regressors and regression techniques for hybrid dynamical systems, and for non-viscous dissipation models.

\appendix

	\section{Appendix A --- Derivation of the Equations of Motion}\label{app:equations-derivation}
	In this and the following section we consider systems more general than those considered earlier, and in so doing assume that the reader is familiar with some basic concepts in geometric mechanics and differential geometry: Lie groups, group actions, and principal bundles.
	We refer the reader to \citet{kobayashi1963foundationsV1,marsden1994introduction,lee2013smooth,bloch2015nonholonomic} for the relevant standard definitions related to Lie groups and group actions, and we refer the reader to \citet{kobayashi1963foundationsV1,marsden1991symmetry, marsden2009lectures, bloch2015nonholonomic} for material on bundles.

	We consider a mechanical system on a configuration space $Q$ whose Lagrangian is of the form kinetic minus potential energy.
	We will also consider this system to be subjected to external viscous forcing arising from a Rayleigh dissipation function, and also subjected to an external force exerted by the locomoting body.
	We are interested in the situation that we have a smooth action $\theta\colon G \times Q \to Q$ of a Lie group $G$ on $Q$, such that the Lagrangian, viscous forces, and external force are all symmetric under the action.
	In this case, we say that $G$ is a \concept{symmetry group}.

	In \S \ref{sec:pertS-mech-visc-conn}, we will define some geometric quantities on $Q$ which encode information about the symmetry and the dynamics.
	Working in coordinates induced by a local trivialization, in \S \ref{sec:pertS-equations-local} we derive the equations of motion in terms of these quantities.
	In \S \ref{sec:pertS-reduction-Stokes-limit}, we recall how the equations become governed by the so-called viscous connection in the Stokesian limit \citep{kelly1996geometry, eldering2016role}, which will set the stage for our derivation in \S \ref{sec:pertS-reduction-perturbed-Stokes-regime} of a corrected reduced-order model for the perturbed Stokes regime.

	\subsection{The mechanical and viscous connections}\label{sec:pertS-mech-visc-conn}
	In this section, we define the mechanical and viscous (or Stokes) connections, roughly following \citet{kelly1996geometry}.
	We consider a Lagrangian $L\colon \T Q \to \R$ which is invariant under the lifted action $\D \theta_g$ of $G$ on $\T Q$ (here $\D$ denotes the derivative or pushforward).
	We assume the Lagrangian to be of the form kinetic minus potential energy, where kinetic energy is given by $\frac{m}{2}k$, where $m > 0$ is a dimensionless mass parameter, $k$ is a smooth symmetric bilinear form, and $m k$ is the \concept{kinetic energy metric}.
	In what follows, we assume that $k$ is positive definite when restricted to tangent spaces to $G$ orbits, but \emph{not} necessarily that $k$ is positive definite on all tangent vectors.\footnote{This does not affect any of the following derivations and results.
		However, this generality is merely a convenience ensuring that our results apply to certain idealized examples, e.g., linkages with some links having zero mass (c.f. \S \ref{sec:performance}). Of course such examples are not physical and, e.g., must be supplemented with assumptions to ensure that the massless links have well-defined dynamics.}
	Denoting by $\g$ the Lie algebra of $G$ and $\g^*$ its dual, we define the (Lagrangian) \concept{momentum map} $J\colon \T Q \to \g^*$ via
	\begin{equation}\label{eq:pertS-J-def}
	\langle J(v_q), \xi \rangle = \langle \FL(v_q), \xi_Q(q) \rangle = mk_q(v_q, \xi_Q(q)),
	\end{equation}
	where $v \in \T_q Q$ and $\xi \in \g$.
	Here $\FL\colon \T Q\to \T^* Q$ is the \concept{fiber derivative} of $L$ given by $\FL(v_q)(w_q)\coloneqq \frac{\partial}{\partial s}|_{s=0} L(v_q + $s$ w_q)$,
	and the smooth vector field $\xi_Q$ on $Q$ is the \concept{infinitesimal generator} defined by $\xi_Q(q)\coloneqq \frac{\partial}{\partial s}|_{s=0}\theta_{\exp(s\xi)}(q)$.
	We define the \concept{mechanical connection} $\Gm\colon \T Q \to \g$ via $\Gm(v_q)\coloneqq \I^{-1}(q)J(v_q)$, where $\I(q)\colon \g \to \g^*$ is the \concept{locked inertia tensor} defined via
	\begin{equation}\label{eq:pertS-I-def}
	\langle \I(q) \xi, \eta \rangle \coloneqq \langle \FL(\xi_Q(q)),\eta_Q(q)\rangle = mk_q(\xi_Q(q),\eta_Q(q)),
	\end{equation}
	where $\xi,\eta \in \g$.

	We now follow an analogous procedure to define the viscous connection $\Gv\colon \T Q \to \R$.
	We consider a Rayleigh dissipation function $R \colon \T Q \to \R$ defined in terms of a $G$-invariant smooth symmetric bilinear form $\nu$ on $Q$: $R(v_q)\coloneqq \frac{c}{2}\nu_q(v_q,v_q)$, where $c>0$ is a dimensionless parameter representing the amount of damping or dissipation in the system due to viscous forces.
	As with $k$, we assume that $\nu$ is positive definite when restricted to tangent spaces to $G$ orbits, but \emph{not} necessarily that $\nu$ is positive definite on all tangent vectors.\footnote{This generality simply allows for, e.g., the situation of a linkage in which not all links are subject to viscous forces.}
	The corresponding force field $F_R\colon \T Q \to \T^* Q$ is given by minus the fiber derivative of $R$, $F_R\coloneqq \Fd(-R)$.
	We define a map $K\colon \T Q \to \g^*$, analogous to the momentum map $J$, via
	\begin{equation}\label{eq:pertS-K-def}
	\langle K(v_q), \xi \rangle = \langle F_R(v_q), \xi_Q(q) \rangle = -c\nu_q(v_q, \xi_Q(q)),
	\end{equation}
	where $v \in \T_q Q$ and $\xi \in \g$.
	We define the \concept{viscous connection} or \concept{Stokes connection} $\Gv\colon \T Q \to \g$  via $\Gv(v_q)\coloneqq \V^{-1}(q)K(v_q)$, where $\V(q)\colon \g \to \g^*$ is defined via
	\begin{equation}\label{eq:pertS-V-def}
	\langle \V(q) \xi, \eta \rangle \coloneqq \langle F_R(\xi_Q(q)),\eta_Q(q)\rangle = -c\nu_q(\xi_Q(q),\eta_Q(q)),
	\end{equation}
	where $\xi,\eta \in \g$.

	Using the $G$-invariance of $L$ and $\nu$, a calculation shows that $\Gm$ and $\Gv$ are equivariant with respect to the adjoint action of $G$ on $\g$:
	\begin{equation}\label{eq:connection_equivariance}
	\forall g \in G\colon \Gm \circ \D \theta_g = \Ad_g \circ \Gm, \quad \Gv \circ \D \theta_g = \Ad_g \circ \Gv
	\end{equation}
	Hence if the natural projection $\pi_Q\colon Q \to Q/G$ from $Q$ to the space of orbits $Q/G$ of points in $Q$ is a principal $G$-bundle, then the mechanical and viscous connections $\Gm$ and $\Gv$ are indeed principal connections; this justifies their titles.

	Now in order for our system to move itself through space, we also allow there to be a $G$-equivariant external force $F_E\colon \R \times \T Q\to \T^*Q$ exerted by the locomoting body, subject to the requirement that $F_E$ takes values in the annihilator of $\ker \D \pi_Q$, the distribution tangent to group orbits. 
This requirement reflects the physically reasonable assumption that the locomoting body can exert only ``internal forces'' which directly affect only its shape $r\in Q/G$  (c.f. \citet[Sec.~3.3]{eldering2016role} and \citet[Sec.~4.2]{bloch1996nonholonomic}).
	For future use, we now prove the following
	\begin{Prop}\label{prop:pertS-J-deriv}
		The derivative of $J$ along trajectories of the $G$-symmetric mechanical system is given by
		\begin{equation}
		\dot{J} = K,
		\end{equation}
		making the canonical identifications $\T_J \g \cong \g$.
	\end{Prop}
	\begin{proof}
		We compute in a local trivialization on $\T Q$ induced by a chart for $Q$, so that we may write a trajectory as $(q,\dot{q})$.
		Note that in such local coordinates, $\FL(q,\dot{q})(v_q) = \frac{\partial L(q,\dot{q})}{\partial \dot{q}}v_q$.
		Hence
		\begin{equation}
		\begin{split}
		\langle \dot{J}(q,\dot{q}), \xi\rangle &= \frac{d}{dt}\left (\frac{\partial L(q(t),\dot{q}(t))}{\partial \dot{q}}\xi_Q(q(t))\right)\\
		&= \left(\frac{d}{dt}\frac{\partial L}{\partial \dot{q}}\right)\xi_Q(q) + \frac{\partial L}{\partial \dot{q}}\D\xi_Q(q)\dot{q}\\
		&= \left(\frac{\partial L}{\partial q} + F_R + F_E \right)\xi_Q(q) + \frac{\partial L}{\partial \dot{q}}\D\xi_Q(q)\dot{q},
		\end{split}
		\end{equation}
		where we obtained the last line using $\frac{d}{dt}\frac{\partial L}{\partial \dot{q}} - \frac{\partial L}{\partial q}=F_R + F_E,$ which follows from the Lagrange-d'Alembert principle \citep[p.~8]{bloch2015nonholonomic}.
		Since $F_E$ annihilates tangent vectors to group orbits, $\langle F_E, \xi_Q(q)\rangle = 0$.
		Hence rearranging and letting $\Phi_\xi^s$ denote the flow of $\xi_Q$, we find
		\begin{equation*}
		\begin{split}
		\langle \dot{J}(q,\dot{q}), \xi \rangle&= \frac{\partial}{\partial s} L\left(\Phi_\xi^s(q(t)), \D \Phi_\xi^s(q(t))\dot{q}(t)\right) + \langle F_R(q,\dot{q}), \xi_Q(q)\rangle\\
		&= \frac{\partial}{\partial s} L\left(\Phi_\xi^s(q(t)), \D \Phi_\xi^s(q(t))\dot{q}(t)\right) + \langle K(q,\dot{q}), \xi \rangle.
		\end{split}
		\end{equation*}
		The derivative term is zero due to the invariance of $L$ under the action of $G$, so from the arbitrariness of $\xi \in \g$ we obtain the desired result.
	\end{proof}
	As a corollary, we obtain a slight generalization of the classical Noether's theorem.
	\begin{Co}[Noether's theorem]\label{co:pertS-Noether}
		Consider a mechanical system given by a $G$-invariant Lagrangian of the form kinetic minus potential energy.
		Assume that the only external forces take values in the annihilator of the distribution tangent to the $G$ orbits.
		Then the derivative of the momentum map $J$ along trajectories satisfies
		\begin{equation*}
		\dot{J}=0.
		\end{equation*}
	\end{Co}
	\begin{proof}
		Set $K = 0$ in Proposition \ref{prop:pertS-J-deriv}.
	\end{proof}

	\subsection{Local form of the equations of motion}\label{sec:pertS-equations-local}
	Assuming that the action of $G$ on $Q$ is free and proper \citep[Ch.~21]{lee2013smooth} so that $\pi_Q\colon Q\to Q/G$ is a principal $G$-bundle, we now derive the equations in a local trivialization, following \citep{kelly1996geometry}.
	In a local trivialization $U \times G$, $\pi_Q$ simply becomes projection onto the first factor and the $G$ action is given by left multiplication on the second factor.
	We define $S \coloneqq Q/G$ to be the \concept{shape space} representing all possible shapes of a locomoting body, and we write a point in the local trivialization as $(r,g)\in U\times G$ where $U\subset S$.
	We assume that $U$ is the domain of a chart for $S$, so that we have induced coordinates $(r,\dot{r})$ for $\T U$.

	Defining the \concept{body velocity}\footnote{
		As mentioned in the main text, the body velocity is often written $g^{-1}\dot{g}$ by an abuse of notation which is only defined on matrix Lie groups where the product of a tangent vector and a group element is naturally defined. %
		We use the alternative notation $\bv{g}$ as a matter of personal preference.}
	$\bv{g} \coloneqq \D\mathrm{L}_{g^{-1}}\dot g$,	the equivariance property \eqref{eq:connection_equivariance} of the connection forms $\Gm, \Gv$
	imply that they may be written in the trivialization as
	\begin{equation}\label{eq:pertS-local-conn-def}
	\begin{split}
	\Gm(r,g)\cdot(\dot{r},\dot{g}) &= \Ad_g\left(\bv{g}+\Am(r)\cdot \dot{r}\right)\\
	\Gv(r,g)\cdot(\dot{r},\dot{g}) &= \Ad_g\left(\bv{g}+\Av(r)\cdot \dot{r}\right),
	\end{split}
	\end{equation}
	where $\Am\colon \T U \to \g$ and $\Av \colon \T U \to \g$ are respectively the \concept{local mechanical connection} and \concept{local viscous connection}.
	We define a diffeomorphism $(r,\dot{r},g,\dot{g})\mapsto (r,\dot{r},g,p)$, with $p$ the \concept{body momentum} defined by
	\begin{equation}\label{eq:pertS-p-def}
	p\coloneqq \Ad_g^*{J} \in \g^*.
	\end{equation}
	Here $\Ad_g^*$ is the dual of the adjoint action $\Ad_g$ of $G$ on $\g$.
	We additionally define
	\begin{equation}\label{eq:pertS-Iloc-Vloc-def}
	\begin{split}
	\Il &\coloneqq \Ad_g^* \I \Ad_g \colon \g \to \g^*\\
	\Vl &\coloneqq \Ad_g^* \V \Ad_g \colon \g \to \g^*
	\end{split}
	\end{equation}
	to be the local forms of $\I$ and $\V$.
	We note that the invariance of the Lagrangian $L$ and Rayleigh dissipation function $R$ under $G$, together with the general identity $\D \theta_g \xi_{Q}(q) = (\Ad_g\xi)_{Q}(\theta_{g}(q))$, imply that $\Il(r),\Vl(r)$ depend on the shape variable $r$ only.

	Rearranging \eqref{eq:pertS-local-conn-def}, using the expressions \eqref{eq:pertS-p-def}, \eqref{eq:pertS-Iloc-Vloc-def}, and using Proposition \ref{prop:pertS-J-deriv}, we obtain the equations of motion
	\begin{equation}\label{eq:pertS-equations-motion-local}
	\begin{split}
	\bv{g} &= -\Am\cdot \dot{r} + \Il^{-1}p\\
	\dot{p} &= \Vl(\Av- \Am)\cdot \dot{r} + \Vl\Il^{-1} p + \ad^*_{\Il^{-1}p}p-\ad^*_{\Am\cdot \dot{r}}p,
	\end{split}
	\end{equation}
	where we have suppressed the $r$-dependence of $\Am,\Av,\Il,\Vl$ for readability.
	Notice that the $\dot{p}$ equation is completely decoupled from $g$.

	In this paper, we are interested in the effect of shape changes on body motion, and not on the generation of shape changes themselves.
	Hence we have suppressed the equations for $\dot{r},\ddot{r}$ from \eqref{eq:pertS-equations-motion-local}, simply viewing $r, \dot{r}$ as inputs in those equations, but see \citet{bloch1996nonholonomic} for more details on the specific form of the equations.
	We merely note that, if the kinetic energy metric is positive-definite, then the Lagrangian is hyperregular and our assumption of $G$-equivariance of the exerted force $F_E$ implies that
	\begin{equation}\label{eq:r-dynamics-form}
	\ddot{r} = f(t,r,\dot{r},\Il^{-1} p)
	\end{equation}
	for some function $f$ which depends on the local trivialization.
	If the kinetic energy metric is not positive-definite (for use in toy examples like those in \S \ref{sec:performance}; see the precise assumptions in \S \ref{sec:pertS-mech-visc-conn}, and the footnote there), then we \emph{assume} that $\ddot{r}$ is given by \eqref{eq:r-dynamics-form}.

	\subsection{Reduction in the Stokesian limit}\label{sec:pertS-reduction-Stokes-limit}
	From the definitions \eqref{eq:pertS-I-def}, \eqref{eq:pertS-V-def} of $\Il, \Vl$, we see that we may define $\bIl, \bVl$ by
	\begin{equation*}
	\Il(r) \eqqcolon m \bIl(r) ~~~~
	\Vl(r) \eqqcolon c \bVl(r).
	\end{equation*}
	Defining the dimensionless parameter $\epsilon\coloneqq \frac{m}{c}$ and multiplying both sides of \eqref{eq:pertS-equations-motion-local} by $\Il\Vl^{-1}$, we obtain the rewritten equations of motion
	\begin{equation}\label{eq:pertS-equations-motion-local-w-epsilon}
	\begin{split}
	\bv{g} &= -\Am\cdot \dot{r} + \frac{1}{m}\bIl^{-1}p\\
	\epsilon \bIl\bVl^{-1}\dot{p} &= m\bIl(\Av- \Am)\cdot \dot{r} + p + \epsilon \bIl\bVl^{-1} \ad^*_{\Il^{-1}p}p- \epsilon \bIl\bVl^{-1}\ad^*_{\Am\cdot \dot{r}}p.
	\end{split}
	\end{equation}
	In considering the limit in which viscous forces dominate the inertia of the locomoting body, \citet{kelly1996geometry} formally set $\epsilon =0$ in \eqref{eq:pertS-equations-motion-local-w-epsilon} to obtain $p = m\bIl(\Am-\Av)\cdot \dot{r}$ from the second equation.
	Substituting this into the first equation of \eqref{eq:pertS-equations-motion-local-w-epsilon}, they derive the following form of the equations of motion:
	\begin{equation}\label{eq:pertS-visc-kelly-limit}
	\bv{g}=-\Av \cdot \dot{r}.
	\end{equation}
	In the language of differential geometry, \eqref{eq:pertS-visc-kelly-limit} states that in the Stokesian limit trajectories are \concept{horizontal} with respect to the viscous connection.
	We will see in the next section that this reduction can be extended away from the $\epsilon \to 0$ limit.

	\section{Appendix B --- Reduction in the Perturbed Stokes Regime}\label{sec:pertS-reduction-perturbed-Stokes-regime}\label{app:reduction-pert-stokes}

	In \citet{eldering2016role}, the argument of \citet{kelly1996geometry} was explained in more detail using the theory of normally hyperbolic invariant manifolds (NHIMs) in the context of geometric singular perturbation theory \citep{fenichel1979geometric,jones1995geometric,kaper1999systems}.
	The idea is to show that for $\epsilon > 0$ sufficiently small,
	the dynamics \eqref{eq:pertS-equations-motion-local-w-epsilon} possess an exponentially attractive invariant \concept{slow manifold} $M_\epsilon$, such that the dynamics restricted to $M_\epsilon$ approach \eqref{eq:pertS-visc-kelly-limit} as $\epsilon \to 0$.
	We give an alternative argument which yields a result differing from that of \citet{eldering2016role} in two ways.
	\begin{enumerate}
		\item \citet{eldering2016role} give an argument for general mechanical systems without symmetry under the assumption that the configuration space $Q$ is compact, although they do indicate that compactness can be replaced with uniformity conditions using noncompact NHIM theory \citep{eldering2013normally}. Our argument assumes symmetry but allows $G$ to be noncompact, though we do require that $S\coloneqq Q/G$ be compact.
		This enables application of our result to locomotion systems with noncompact symmetry groups, such as the Euclidean group of planar rigid motions $\SE(2)$ as in the systems of \S \ref{sec:performance}.

		\item \citet{eldering2016role} consider the limit $m\to 0$ while holding $c$ and the force exerted by the locomoting body fixed.
		This makes sense, because if the exerted force were held fixed while taking $c \to \infty$, then trivial dynamics would result in the singular limit: the system would not move at all.
		Rather than holding the exerted force fixed, we will consider the differential equation prescribing the \emph{dynamics} of the shape variable to be fixed.\footnote{This implicitly assumes that the locomoting body is capable of exerting $\bo(c)$ forces.}
		Under this assumption, we show that the dynamics depend only on the \emph{ratio} $\epsilon = \frac{m}{c}$, and in particular the dynamics obtained in the two singular limits $m \to 0$ and $c \to \infty$ are the same.
	\end{enumerate}
	Before stating Theorem \ref{th:pertS_NAIM_persists}, we need the following definition.
	\begin{Def}[$C^k_b$ time-dependent vector fields]\label{def:ckb-vector-fields}
		Let $M$ be a compact manifold with boundary, and let $f\colon \R \times M \to \T M$ a $C^{k \geq 0}$ time-dependent vector field.
		Let $(U_i)_{i=1}^n$ be a  finite open cover of $M$ and $(V_i, \psi_i)_{i=1}^n$ be a finite atlas for $M$ such that $\bar{U}_i \subset V_i$ for all $i$, and for each $i$ define $f_i \coloneqq (\D \psi_i \circ f\circ (\id_\R \times \psi_i^{-1}))$.
		We define an associated $C^k$ norm $\norm{f}_k$ of $f$ via
		\begin{equation}
		\norm{f}_{ k}\coloneqq \max_{1\leq i\leq n}\max_{\substack{0 \leq j \leq k\\x\in \psi_i(\bar{U}_i)}}\norm{\D^j f_i(x)},
		\end{equation}
		where $\norm{\D^j f_i(x)}$ denotes the norm of a $j$-linear map; here $\D^j f$ includes partial derivatives with respect to time as well as the spatial variables.
		If $\norm{f}_k< \infty$, we say that $f$ is $C^k$-bounded and write $f\in C^k_b$.
		The norm $\norm{\cdot}_k$ makes the $C^k_b$ time-dependent vector fields into a Banach space.
		The norms induced by any two such finite covers of $M$ are equivalent, and thereby induce a canonical \concept{$C^k_b$ topology} on the space of $C^k_b$ time-dependent vector fields.
	\end{Def}
	\begin{Rem}
		Definition
		\ref{def:ckb-vector-fields} defines the $C^k_b$ topology on the space of $C^k_b$ time-dependent vector fields on a compact manifold.
		As discussed in \citet[Sec.~1.7]{eldering2013normally},
		this $C^k_b$ topology is finer than the $C^k$ weak Whitney topology and coarser than the $C^k$ strong Whitney topology \citep[Ch.~2]{hirsch1976differential}, but all of these topologies induce the same topology on the subspace of time-independent vector fields due to compactness.
		Definition \ref{def:ckb-vector-fields} is a special case of the definition in \citet[Ch.~2]{eldering2013normally} for the $C^k_b$ topology on $C^k_b$ vector fields on Riemannian manifolds of bounded geometry, and on $C^k_b$ maps between such manifolds.
	\end{Rem}

	The following theorem concerns a $G$-symmetric dynamical system on $\T Q$ whose equations of motion are consistent with our assumptions so far: i.e., they are given in local trivializations by \eqref{eq:pertS-equations-motion-local-w-epsilon} and an equation of the form \eqref{eq:r-dynamics-form}.

	\begin{Th}\label{th:pertS_NAIM_persists}
		Assume that $S=Q/G$ is compact.
		Let $2 \leq k < \infty$, and let $X^\epsilon$ be a $C^k$ family of $G$-symmetric time-dependent vector fields on $\T Q$ with the following properties:

		\begin{enumerate}
			\item For every compact neighborhood with $C^k$ boundary $K_0 \subset \T Q$ and $\epsilon > 0$, $X^\epsilon|_{\R \times K_0}\in C^k_b$ (Definition \ref{def:ckb-vector-fields}).

			\item There exists a compact connected neighborhood $K\subset \T S$ of the zero section of $\T S$ with $C^k$ boundary, such that $N\coloneqq \D \pi_Q^{-1}(K) \subset \T Q$  is positively invariant for $X^\epsilon$, for all sufficiently small $\epsilon > 0$.

			\item $X^\epsilon$ is given in each local trivialization $\T(U\times G)$, where $U$ is a chart for $S$, by \eqref{eq:r-dynamics-form} and \eqref{eq:pertS-equations-motion-local-w-epsilon}:
			\begin{equation}\label{eq:reduced-system}
			\begin{split}
			\ddot{r} &= f\left(t,r,\dot{r},\frac{1}{m}\bIl^{-1}p\right)\\
			\epsilon \bIl\bVl^{-1}\dot{p} &= m\bIl(\Av- \Am)\cdot \dot{r} + p + \epsilon \bIl\bVl^{-1} \ad^*_{\Il^{-1}p}p- \epsilon \bIl\bVl^{-1}\ad^*_{\Am\cdot \dot{r}}p\\
			\bv{g} &= -\Am\cdot \dot{r} + \frac{1}{m}\bIl^{-1}p
			\end{split}
			\end{equation}
			for some function $f$ which depends on the local trivialization but is independent of $\epsilon$.
		\end{enumerate}
		Then for all sufficiently small $\epsilon > 0$, there exists a $C^k$ noncompact normally hyperbolic invariant manifold with boundary $M_\epsilon\subset \R \times N \subset \R \times \T Q$ for the extended dynamics given by the extended vector field $(1,X_\epsilon)$ on $\R \times \T Q$.
		Additionally, $M_\epsilon$ is uniformly (in time and space) globally asymptotically stable and uniformly locally exponentially stable (with respect to the distance induced by any complete $G$-invariant Riemannian metric on $\T Q$) for the extended dynamics restricted to $\R \times N$.
		Finally, there exists $\epsilon_0 > 0$ such that, for each local trivialization $U\times G$, there exists a $C^k$ map $h_\epsilon\colon \R \times (\T U \cap K) \times (0,\epsilon_0)\to \g^*$
		such that $M_\epsilon \cap \D \pi_Q^{-1}(\T U \cap K)$ corresponds to
		\begin{equation}\label{eq:pertS-Meps-graph-lagrangian}
		\{(t,r,\dot{r},p,g): p =  h_\epsilon(t,r,\dot{r},\epsilon)\},
		\end{equation}
		\begin{equation*}
		h_\epsilon(t,r,\dot{r},\epsilon) = \Il\left[(\Am(r)-\Av(r))\cdot \dot{r} + \bo(\epsilon)\right]
		\end{equation*}
		(with $p$ defined by \eqref{eq:pertS-p-def}), and $h_\epsilon$ together with its partial derivatives of order $k$ or less are bounded uniformly in time.
		If $f(t,r,\dot{r},\Il^{-1} p)$ is independent of $t$, then $h_\epsilon$ and $M_\epsilon$ are independent of $t$, and $M_\epsilon$ can be interpreted as a compact NHIM for the (non-extended) dynamics restricted to $N$.
	\end{Th}

	\begin{Rem}
		Note that even if we assume $f\in C^\infty$, we can generally only obtain $C^k$ NHIMs $M_\epsilon$ for $k$ finite.
		This is because we obtain $M_\epsilon$ as a perturbation of a NHIM $M_0$, and perturbations of $C^\infty$ NHIMs are generally only finitely smooth because the maximum perturbation size $\epsilon$ required to obtain degree of smoothness $k$ for $M_\epsilon$ generally depends on $k$ in such a way that $\epsilon\to 0$ as $k\to \infty$.
		See \citet[Rem.~1.12]{eldering2013normally} and \citet{van1979center} for more discussion.
	\end{Rem}

	\begin{Rem}
		By replacing compactness of $Q/G$ with uniformity conditions, it should be possible to generalize Theorem \ref{th:pertS_NAIM_persists} to the situation of $Q$ noncompact where either $Q/G$ is noncompact, or where there is no symmetry at all.
		This was pointed out in \citet[App.~1]{eldering2016role}.
		This observation seems important for the consideration of dissipative mechanical systems which are only \emph{approximately} symmetric under a group $G$, which seems to be a more realistic assumption.
	\end{Rem}
	\begin{Rem}\label{rem:recover-kelly-eqn}
		By taking $\epsilon \to 0$ in Theorem \ref{th:pertS_NAIM_persists}, we find that $p = \Il(\Am-\Av)\cdot \dot{r}$ in the limit.
		Substituting this into the first equation of \eqref{eq:pertS-equations-motion-local-lagrangian}, we obtain Equation \eqref{eq:pertS-equations-motion-local} as in \citet{kelly1996geometry}.
	\end{Rem}

	\begin{proof}~\\
		\textit{Preparation of the equations of motion}. Throughout the proof, we consider the dynamics in local trivializations of the form $U\times G$ for $Q$, where $U$ is the domain of a chart for $S$, so that we have induced coordinates $(r,\dot{r})$ for $\T U$.
		In such a local trivialization we would like to use \eqref{eq:reduced-system} to analyze the dynamics, but there are two (related) problems with this.
		First, the definition of $p$ depends on $m$, and this will cause difficulties in verifying Definition \ref{def:ckb-vector-fields} to check that certain vector fields are close in the $C^k_b$ topology.
		Second, we would like to analyze \eqref{eq:reduced-system} in a singular perturbation framework, but this is difficult to do directly because $m$ explicitly appears, and the size of $m$ may or may not be commensurate with the size of $\epsilon$.
		To remedy this situation, we change variables via the diffeomorphism $(r,\dot{r},p,g)\mapsto (r,\dot{r},\Omega,g)$ of $\T U \times \g^* \times G \to \T U \times \g\times G$ where $\Omega\in \g$ is defined by
		\begin{equation}\label{eq:pertS-Omega-def}
		\Omega\coloneqq \Il^{-1}p = \Ad_{g^{-1}}\Gm(\dot{g},\dot{r})=\bv{g} + \Am\cdot\dot{r}.
		\end{equation}
		Sometimes $\Omega$ is referred to as the \concept{(body) locked angular velocity} \citep[p.~61]{bloch1996nonholonomic}.
		Differentiating $\Il \Omega = p$, using \eqref{eq:reduced-system}, and rearranging yields
		\begin{equation}\label{eq:pertS-equations-motion-local-lagrangian}
		\begin{split}
		\dot{t} &= 1\\
		\dot{r} &= v\\
		\dot{v} &= f(t,r,v,\Omega)\\
		\epsilon \dot{\Omega} &= -\epsilon\bIl^{-1} \left(\frac{d}{dt}\bIl\right)\Omega + \bIl^{-1}\bVl(\Av- \Am)\cdot v + \bIl^{-1}\bVl\Omega + \epsilon  \bIl^{-1}\ad^*_{\bv{g}}\bIl \Omega,
		\end{split}
		\end{equation}
		where we have introduced the variable $v\coloneqq \dot{r}$.
		We have written $\ad^*_{\bv{g}}$ for space reasons, but note that the $\dot{\Omega}$ equation is independent of $g$ since
		\begin{equation}\label{eq:bv-in-terms-omega}
		\bv{g} = -\Am \cdot \dot r + \Omega,
		\end{equation}
		and this implies that $\ad^*_{\bv{g}} = \ad^*_{\Omega}-\ad^*_{\Am\cdot \dot{r}}$.
		We see that \eqref{eq:pertS-equations-motion-local-lagrangian} is split into slow $(t,r,v)$ and fast $(\Omega)$ variables, which is the appropriate setup for a singular perturbation analysis.
		The remainder of the proof consists of two parts: (i) proving that the NHIM $M_\epsilon$ exists, and (ii) establishing the stability properties of $M_\epsilon$.

		\textit{Proof that $M_\epsilon$ exists}. Introducing the ``fast time'' $\tau \coloneqq \frac{1}{\epsilon} t$ and denoting a derivative with respect to $\tau$ by a prime, after the time-rescaling we obtain the regularized equations
		\begin{equation}\label{eq:pertS-equations-motion-local-lagrangian-fast-time}
		\begin{split}
		t' &= \epsilon\\
		r' &= \epsilon v\\
		v' &= \epsilon f(t,r,v,\Omega)\\
		\Omega' &= -\epsilon\bIl^{-1} \left(\frac{d}{dt}\bIl\right)\Omega + \bIl^{-1}\bVl(\Av- \Am)\cdot v + \bIl^{-1}\bVl\Omega + \epsilon  \bIl^{-1}\ad^*_{\bv{g}}\bIl \Omega.
		\end{split}
		\end{equation}
		This rescaling of time is equivalent to replacing the vector field $(1,X_\epsilon)$ on $\R \times \T Q$ by $(\epsilon,\epsilon X_\epsilon)$.
		We see from \eqref{eq:bv-in-terms-omega} and \eqref{eq:pertS-equations-motion-local-lagrangian-fast-time} that there is a well-defined $C^k$ time-dependent vector field $\tilde{X}_0$ given by the pointwise limit $\tilde{X}_0\coloneqq \lim_{\epsilon\to 0} \epsilon X_\epsilon$.
		Given any $G$-symmetric time-dependent vector field $Y$ on $\T Q$, we let $Y/G$ denote the corresponding reduced vector field on $(\T Q)/G$.
		Hence \eqref{eq:pertS-equations-motion-local-lagrangian-fast-time} shows that the extended vector field $(1,\tilde{X}_0/G)$ has a smooth embedded submanifold $(M_0/G)$ of critical points whose intersection with a locally trivializable neighborhood is given by
		\begin{equation}\label{eq:M_0/G-graph}
		\{(r,v,\Omega) \in \T U \times \g: \Omega = (\Am-\Av)\cdot v\},
		\end{equation} and it is readily seen that $M_0/G$ is described globally as the quotient of the Ehresmann connection $M_0 \coloneqq \ker \Gv$ by the lifted action of $G$ on $\T Q$.

		Furthermore, $M_0/G$ is a globally exponentially stable NHIM for the $\epsilon = 0$ system.
		To see this, first note that in any local trivialization $t, r, v$ are constants when $\epsilon = 0$, and hence $\Omega '$ is of the form $\Omega ' = \bIl^{-1} \bVl \Omega + b$ for a constant $b$, and therefore has a globally exponentially stable equilibrium provided that all eigenvalues of $\bIl^{-1}\bVl$ have negative real part. To see that this is the case, fix a basis of $\g$ and corresponding dual basis for $\g^*$, and first consider the product $\I^{-1} \V$. With respect to our chosen basis, $\I, \V$ and their inverses $\I^{-1}, \V^{-1}$ are respectively represented by $r$-dependent matrices $I_{ij}, V_{ij}$ and their inverses $I^{ij}, V^{ij}$.
		It is immediate from the definitions \eqref{eq:pertS-I-def} and \eqref{eq:pertS-V-def} that $I_{ij}$ and $V_{ij}$ are respectively positive definite and negative definite symmetric matrices (this is why we required the bilinear forms $k, \nu$ to be positive definite when restricted to vectors tangent to $G$ orbits).
		Since $I_{ij}$ is symmetric positive definite, we may let $(\sqrt{I})_{ij}$ be a matrix square root of $I_{ij}$ and let $(\sqrt{I})^{ij}$ be its inverse.
		But then the product $I^{ik}V_{kj}$ is similar to the symmetric negative definite matrix $(\sqrt{I})^{ik}V_{k\ell}(\sqrt{I})^{\ell j}$ (Einstein summation implied).
		Hence $\I^{-1} \V$ has only eigenvalues with negative real part, and the same is true of $\Il^{-1}\Vl$ because of the similarity $\Il^{-1}\Vl = \Ad_g^{-1} \I^{-1}\V \Ad_g$.

		Let $\tilde{\pi}\colon (\T Q)/G \to \T S$ denote the projection induced by $\D \pi_Q$.
		Equation \eqref{eq:M_0/G-graph} implies that $M_0/G$ is the image of a section $\sigma_0\colon \T S \to (\T Q)/G$ of $\tilde{\pi}$.
		Hence $(M_0/G)\cap \tilde{\pi}^{-1}(K) = \sigma_0(K)$ is compact, and $M_0/G$ intersects $\tilde{\pi}^{-1}(\partial K)$ transversely.
		Furthermore, the assumption that $X^\epsilon|_{\R \times K_0} \in C^k_b$ for any compact neighborhood with $C^k$ boundary $K_0 \subset \T Q$ implies that all partial derivatives of $f$ are bounded on compact sets uniformly in time.
		This makes it clear that for any compact $K_1\subset (\T Q)/G$,  $(\epsilon X_\epsilon/G)|_{\R \times K_1}$ can be made arbitrarily close to $(\tilde{X}_0/G)|_{\R \times K_1}$ in the $C^k_b$ topology (Definition \ref{def:ckb-vector-fields}) by taking $\epsilon > 0$ sufficiently small.
		Hence by the noncompact NHIM results of \citet[Sec. 4.1-4.2]{eldering2013normally}, it follows that $(M_0/G)\cap \tilde{\pi}^{-1}(K)$ persists in extended state space $\R \times N$ to a nearby attracting NHIM $M_\epsilon/G$ with boundary for $(\epsilon, \epsilon X_\epsilon/G)$.\footnote{$M_\epsilon/G$ is unique up to the choice of a cutoff function used to modify the dynamics near the boundary of a slightly enlarged neighborhood of $\tilde{\pi}^{-1}(K)$, used in order to render a slightly enlarged version of $(M_0/G)\cap \tilde{\pi}^{-1}(K)$ overflowing invariant \citep[Sec. 4.3]{eldering2013normally}. See \citet[Sec. 5]{kvalheim2018global} and \citet[Sec. 2]{josic2000synchronization} for more details on such boundary modifications.}
		Furthermore, $M_\epsilon /G$ is the image of a section $\sigma_\epsilon\colon \R \times K \to (\T Q)/G$ of $\tilde{\pi}$, and is given in each local trivialization of $(\T Q)/G$ by the graph of a function $\Omega = \tilde{h}_\epsilon(t,r,\dot{r},\epsilon)$ which is $C^k$ bounded uniformly in time.
		By symmetry, the preimage $M_\epsilon = \pi_{\T Q}^{-1}(M_\epsilon/G)$ of  $M_\epsilon/G$  via the quotient $\pi_{\T Q}\colon \T Q \to (\T Q)/G$ yields a NHIM $M_\epsilon$ for $(\epsilon, \epsilon X_\epsilon)$ (and hence also for $(1,X_\epsilon)$) on the subset $\R \times N$ of $\R \times \T Q$, and $M_\epsilon$ is given in each local trivialization by the graph of the same function $\Omega = \tilde{h}_\epsilon$ as $M_\epsilon/G$ but augmented with trivial dependence on $g$.
		The function $h_\epsilon$ from the theorem statement is given by $h_\epsilon = \Il \tilde{h}_\epsilon$.

		\textit{Proof of the stability properties of $M_\epsilon$}.
		Fix any complete $G$-invariant Riemannian metric on\footnote{For example, take the Sasaki metric on $\T Q$ induced by any complete $G$-invariant metric on $Q$.} $\T Q$, so that it descends to a metric on $(\T Q)/G$ making $\pi_{\T Q}\colon\T Q \to (\T Q)/G$ into a Riemannian submersion \citep[p.~185]{docarmo1992riemannian}.
		We have distance functions $\tilde{d}$ and $d$ on $\T Q$ and $(\T Q)/G$ induced by these metrics.
		For $t \in \R$, we let $M_\epsilon(t)\coloneqq M_\epsilon \cap (\{t\}\times N)$ and $M_\epsilon(t)/G\coloneqq \pi_{\T Q}(M_\epsilon(t))$.
		Given $w\in \T Q$ and its orbit $\pi_{\T Q}(w) \in (\T Q)/G$, it follows that for all $t\in \R$, $\tilde{d}(w, M_\epsilon(t)) = d(\pi_{\T Q}(w), M_\epsilon(t)/G)$.\footnote{To prove this, first note that $d(\pi_{\T Q}(w), M_\epsilon(t)/G)\leq \tilde{d}(w, M_\epsilon(t))$ because the length $\ell(\tilde{\gamma})$ of any curve $\tilde{\gamma}\colon[0,1]\to \T Q$ satisfies $\ell(\pi_{\T Q}\circ \tilde{\gamma})\leq \ell(\tilde{\gamma})$. But if $\gamma:[0,1]\to (\T Q)/G$ is any curve joining $\pi_{\T Q}(w)$ to $M_\epsilon/G$, then its horizontal lift $\tilde{\gamma}$ is a curve joining $w$ to $M_\epsilon$ such that $\ell(\tilde{\gamma})=\ell(\gamma)$. Taking the infimum over all such $\gamma$ shows that $\tilde{d}(w, M_\epsilon(t)) = d(\pi_{\T Q}(w), M_\epsilon(t)/G)$.}
		Hence it suffices to prove that $M_\epsilon/G$ is uniformly globally asymptotically stable and locally exponentially stable for the vector field $(1, X_\epsilon/G)$ on $\R \times \tilde{\pi}^{-1}(K) = \R \times \pi_{\T Q}(N)$, and to do this it suffices to prove the same for $(\epsilon, \epsilon X_\epsilon/G)$.

		Fixing an inner product $\langle \slot, \slot \rangle$ and associated norm $\norm{\slot}$ on $\g$, we accomplish this in two steps.
		First, we show that there exists a compact neighborhood $K_0 \subset \pi_{\T Q}(N)$ of $M_\epsilon/G$ such that $K_0$ is positively invariant for the time-dependent flow of $X_\epsilon$, and such that any other compact neighborhood $K_1\subset \pi_{\T Q}(N)$ of $M_\epsilon/G$ flows into $K_0$ after some finite time depending on $K_1$ but independent of the initial time.
		Second, we show that all trajectories in $K_0$ converge to $M_\epsilon/G$ at a uniform exponential rate.
		To achieve this second step, we show that in the intersection of each local trivialization with $K_0$, $\norm{\Omega-\tilde{h}_\epsilon(t,r,v)}$ decreases at an exponential rate.
		Since $(\T Q)/G$ is covered by finitely many local trivialization (by compactness of $S$), and since all Riemannian metrics are uniformly equivalent on compact sets\footnote{Let $\norm{\slot}, \norm{\slot}'$ denote the Finslers (norms) induced by two Riemannian metrics, and $K_0$ our compact set. Since all norms are equivalent on finite-dimensional vector spaces, we have that the restrictions of these norms to the tangent space of a single point $x$ satisfy $\frac{1}{c(x)}\norm{\slot} \leq \norm{\slot}' \leq c(x) \norm{\slot}$. Defining $\bar{c}\coloneqq \sup_{x\in K_0}c(x)$, we obtain the uniform equivalence $\frac{1}{\bar{c}}\norm{\slot} \leq \norm{\slot}' \leq \bar{c} \norm{\slot}$ on all of $K_0$. If $K_0$ is a connected submanifold and we give it the restricted metrics, then by considering the lengths of curves in $K_0$ this implies the uniform bound $\frac{1}{\bar{c}}d \leq d' \leq \bar{c}d$ on the Riemannian distances between points in $K_0$ with respect to the restricted metrics.}, this will establish uniform exponential convergence of points in $K_0$ with respect to the distance induced by any Riemannian metric, and in particular the distance $d$.

		Consider a local trivialization $U\times G$ of $Q$ and the associated form \eqref{eq:pertS-equations-motion-local-lagrangian-fast-time} of the dynamics restricted to $\tilde{\pi}^{-1}(K \cap \T U)$.
		Differentiating $\norm{\Omega}^2$ using the last equation of \eqref{eq:pertS-equations-motion-local-lagrangian-fast-time}, it is easy to check that $\frac{d}{d\tau}\norm{\Omega}^2 \to -\infty$ as $\norm{\Omega}^2 \to \infty$, uniformly in $(t,r,v,\epsilon)$ for $\epsilon$ sufficiently small.
		(This follows from the negative definiteness of $\Il^{-1}\Vl$ and the compactness of $K$.)
		Hence we see that there exists $k_0 > 0$ such that for all $\epsilon$ sufficiently small, $\frac{d}{d\tau}\norm{\Omega}^2 \leq -1$ when $\norm{\Omega}^2 \geq k_0^2$.
		Now $k_0$ might depend on the local trivialization, but we can replace $k_0$ with the largest such constant selected from finitely many fixed local trivializations covering $Q$.
		Hence there exists a compact subset $K_0 \subset \pi_{\T Q}(N)$ given by $\{\norm{\Omega} \leq k_0\}$ in each of these fixed local trivializations, such that $K_0$ is positively invariant for the time-dependent flow of $X_\epsilon$ and such that any other compact neighborhood $K_1 \subset \pi_{\T Q}(N)$ of $M_\epsilon/G$ flows into $K_0$ after some finite time independent of the initial time.

		It remains only to establish the uniform exponential rate of convergence of trajectories in $K_0$ to $M_\epsilon$.
		For each local trivialization $U \times G$ of $Q$, we define the translated variable $\tilde{\Omega}\coloneqq \Omega - \tilde{h}_\epsilon(t,r,v,\epsilon)$.
		Since $M_\epsilon/G$ is invariant, we must have $\tilde{\Omega}' = 0$ whenever $\tilde{\Omega} = 0$.
		Differentiating $\tilde{\Omega}$ using \eqref{eq:pertS-equations-motion-local-lagrangian-fast-time}, we therefore find that
		\begin{equation}\label{eq:omega-tilde-equation}
		\begin{split}
		\tilde{\Omega}' &=  \left[-\epsilon\bIl^{-1} \left(\frac{d}{dt}\bIl\right) + \epsilon  \bIl^{-1}\ad^*_{\bv{g}}\bIl + \epsilon \zeta(t,r,v,\tilde{\Omega}) + \bIl^{-1}\bVl \right]\tilde{\Omega}\\
		&\eqqcolon \left[\epsilon A(t,r,v,\tilde{\Omega}) + \bIl^{-1}\bVl(r) \right]\tilde{\Omega},
		\end{split}
		\end{equation}
		since all of the terms which do not vanish when $\tilde{\Omega} = 0$ must cancel.
		Here $\zeta$ is defined via Hadamard's lemma \citep[Lemma~2.8]{nestruev2003smooth}:
		\begin{equation}
		\zeta(t,r,v,\tilde{\Omega}) \coloneqq \frac{\partial}{\partial v}\tilde{h}_\epsilon(t,r,v) \int_{0}^{1}\frac{\partial}{\partial \Omega}f(t,r,v,\tilde{h}_\epsilon(t,r,v) + s \tilde{\Omega})\,ds,
		\end{equation}
		so that $\zeta(t,r,v,\tilde{\Omega})\tilde{\Omega} = \tilde{h}_\epsilon(t,r,v) f(t,r,v,\tilde{h}_\epsilon + \tilde \Omega)$.
		As previously mentioned, the $C^k$ boundedness of $X_\epsilon$ on compact subsets of $\T Q$ implies that $\tilde{h}_\epsilon$, $f$, and their first $k$ partial derivatives are uniformly bounded on sets of the form $\R \times K_2$ with $K_2$ compact.
		Hence whenever $\norm{\Omega}\leq k_0$ and $(r,v) \in U \cap K$, $\norm{A(t,r,v,\tilde{\Omega})} \leq L$ for some constant $L$ depending on the local trivialization; we replace $L$ with the largest such constant chosen from finitely many local trivializations covering $Q$.
		Integrating both sides of
		\eqref{eq:omega-tilde-equation}, taking norms using the triangle inequality, and applying Gr\"{o}nwall's Lemma therefore yields
		\begin{equation}\label{eq:gas-gronwall-estimate}
		\begin{split}
		\norm{\tilde{\Omega}(\tau)} &\leq e^{-\lambda (\tau-\tau_0)} e^{\int_{\tau_0}^{\tau}\epsilon\norm{A(t(s),r(s),v(s),\tilde{\Omega}(s)}\,ds }\norm{\tilde{\Omega}(\tau_0)}\\
		& \leq e^{\left[-\lambda + \epsilon L \right](\tau-\tau_0)} \norm{\tilde{\Omega}(\tau_0)}.
		\end{split}
		\end{equation}
		where $-\lambda < 0$ is defined via $-\lambda\coloneqq \sup_{r \in S} \max\,  \textnormal{spec}(\bIl^{-1}\bVl(r))$, and is strictly negative since $S$ is compact.
		By the previous discussion, requiring  $\epsilon > 0$ to be sufficiently small so that $-\lambda + \epsilon L < 0$ completes the proof.
	\end{proof}

	Theorem \ref{th:pertS_NAIM_persists} and Remark \ref{rem:recover-kelly-eqn} show that, to zeroth order in $\epsilon$, the dynamics restricted to the slow manifold $M_\epsilon$ are given by the viscous connection model \eqref{eq:pertS-visc-kelly-limit}.
	The following theorem shows that the dynamics restricted to $M_\epsilon$ can be explicitly computed to higher order in $\epsilon$.
	We compute the restricted dynamics to first order in $\epsilon$.
	Higher order terms in $\epsilon$ can also be computed recursively, but we choose not to pursue this here.
	\begin{Th}\label{th:pertS-dynamics-on-slow-manifold}
		Assume the same hypotheses as in Theorem \ref{th:pertS_NAIM_persists}.
		Then the dynamics restricted to the slow manifold $M_\epsilon$ are given in a local trivialization by
		\begin{equation}\label{eq:pertS-corrected-eqn-first-order}
		\bv{g}= -\Av\cdot \dot{r} + \epsilon \bVl^{-1} \left(\left(\frac{\partial}{\partial_r} \bar{h}_0\right)\dot{r} + \left(\frac{\partial}{\partial\dot{r}}\bar{h}_0\right) \ddot{r} - \ad^*_{\bv{g}}(\bar{h}_0)\right) + \bo(\epsilon^2),
		\end{equation}
		where
		\begin{equation*}
		\bar{h}_0(r,\dot{r}) \coloneqq \frac{1}{m}h_0(r,\dot{r}) =  \bIl (\Am(r)-\Av(r))\cdot \dot{r},
		\end{equation*}
		where we are using the definition $\bIl \coloneqq \frac{1}{m}\Il$.
		Alternatively, we may write
		\begin{equation}\label{eq:pertS-corrected-eqn-first-order-f}
		\bv{g}= -\Av\cdot \dot{r} + \epsilon \bVl^{-1} \left(\left(\frac{\partial}{\partial r} \bar{h}_0\right)\dot{r} + \left(\frac{\partial}{\partial\dot{r}}\bar{h}_0\right) f(t,r,\dot{r},\bIl^{-1}\bar{h}_0) - \ad^*_{\bv{g}}(\bar{h}_0)\right) + \bo(\epsilon^2),
		\end{equation}
		for a different $\bo(\epsilon^2)$ term.
	\end{Th}
	\begin{Rem}
		Notice the presence, in the second term of \eqref{eq:pertS-corrected-eqn-first-order},  of $\bar{h}_0$ rather than $h_0$ of \eqref{eq:pertS-Meps-graph-lagrangian}. This is important because the expression for $h_0$ contains an $\Il = m \bIl$ factor.
		Because of the possibility that the size of $m$ is commensurate with $\epsilon$, this means that $h_0$ could be $\bo(\epsilon)$.
		However, $\bar{h}_0$ is $\bo(1)$, ensuring that the second term is $\bo(\epsilon)$ but not $\bo(\epsilon^2)$.
	\end{Rem}
	\begin{Rem}
		Equations \eqref{eq:pertS-corrected-eqn-first-order} and \eqref{eq:pertS-corrected-eqn-first-order-f} can be viewed as adding $\bo(\epsilon)$ correction terms to the viscous connection model \eqref{eq:pertS-visc-kelly-limit}, valid in the limit $\epsilon \to 0$, to account for the more realistic situation that the inertia-damping ratio $\frac{m}{c} = \epsilon$ is small but nonzero.
	\end{Rem}
	\begin{proof}[Proof of Theorem \ref{th:pertS-dynamics-on-slow-manifold}]
		Consider the function $$\tilde{h}_\epsilon(t,r,\dot{r},\epsilon) \coloneqq \Il^{-1} h_\epsilon =  (\Am(r)-\Av(r))\cdot \dot{r} + \bo(\epsilon)$$ from the proof of Theorem \ref{th:pertS_NAIM_persists}, and define $\bar{h}_\epsilon\coloneqq \bIl \tilde{h}_\epsilon = \frac{1}{m}h_\epsilon$.
		Since $\bar{h}_\epsilon, \tilde{h}_\epsilon \in C^k$, we may expand them as asymptotic series
		\begin{equation}\label{eq:asympt-series}
		\begin{split}
		\bar{h}_\epsilon &=  \bar{h}_0 + \epsilon \bar{h}_1 + \ldots + \epsilon^{k} \bar{h}_{k} + \bo(\epsilon^{k+1})\\
		\tilde{h}_\epsilon &=  \tilde{h}_0 + \epsilon \tilde{h}_1 + \ldots + \epsilon^{k} \tilde{h}_{k} + \bo(\epsilon^{k+1}),
		\end{split}
		\end{equation}
		where for all $i$, $\bar{h}_i = \bIl \tilde{h}_i$.
		We also already know from Theorem \ref{th:pertS_NAIM_persists} that $\tilde{h}_0 =(\Am-\Av)\cdot\dot{r}$, and therefore $\tilde{h}_0(t,r,\dot{r}) \equiv \tilde{h}_0(r,\dot{r})$ has no explicit $t$-dependence.
		We now compute $\tilde{h}_1$ via a standard technique \citep{jones1995geometric}.
		Differentiating both sides of the equation $\Omega = \tilde{h}_\epsilon(t,r,\dot{r},\epsilon)$ with respect to time (using \eqref{eq:pertS-equations-motion-local-lagrangian} to differentiate the left hand side), substituting the second equation of \eqref{eq:asympt-series} for $\Omega$ in the resulting expression, and retaining terms only up to $\bo(\epsilon)$ we obtain
		\begin{align*}
		-\epsilon\bIl^{-1} \left(\frac{d}{dt}\bIl\right)\tilde{h}_0 + \bIl^{-1}\bVl(\Av- \Am)\cdot \dot{r} + \bIl^{-1}\bVl\left(\tilde{h}_0+\epsilon \tilde{h}_1 \right) + \epsilon  \bIl^{-1}\ad^*_{\bv{g}}\bIl \tilde{h}_0 = \epsilon \dot{\tilde{h}}_0 + \bo(\epsilon^2).
		\end{align*}
		Equating the coefficients of $\epsilon$ yields
		\begin{align*}
		\tilde{h}_1 &= \bVl^{-1}  \left(\frac{d}{dt}\bIl\right)\tilde{h}_0+ \bVl^{-1}\bIl \dot{\tilde{h}}_0 - \bVl^{-1}\ad^*_{\bv{g}}\bIl \tilde{h}_0\\
		&= \bVl^{-1}\frac{d}{dt}\left(\bIl \tilde{h}_0 \right) - \bVl^{-1}\ad^*_{\bv{g}}\bIl \tilde{h}_0.
		\end{align*}
		Since $h_1 = \Il \tilde{h}_1$ and $\bar{h}_0 = \bIl \tilde{h}_0$, we find
		\begin{equation}
		h_1 = \Il\bVl^{-1}\frac{d}{dt}\left(\bar{h}_0 \right) - \Il \bVl^{-1}\ad^*_{\bv{g}}\left(\bar{h}_0\right),
		\end{equation}
		and therefore (substituting $\ddot{r} = f(t,r,\dot{r},\Il^{-1} p) = f(t,r,\dot{r},\tilde{h}_0) + \bo(\epsilon)$ and differentiating $\bar{h}_0(r,\dot{r})$ via the chain rule),
		\begin{equation}\label{eq:pertS-h-eps_first-two-terms}
		\begin{split}
		h_\epsilon(t,r,\dot{r},\epsilon) &= \Il(\Am-\Av)\cdot\dot{r} \\ &  + \epsilon \Il \bVl^{-1} \left(\left(\frac{\partial}{\partial r} \bar{h}_0\right)\dot{r} + \left(\frac{\partial}{\partial\dot{r}}\bar{h}_0\right) f(t,r,\dot{r},\tilde{h}_0) - \ad^*_{\bv{g}}(\bar{h}_0)\right) + \Il \bo(\epsilon^2).
		\end{split}
		\end{equation}
		Notice that, since $\tilde{h}_0$ is a function of $r,\dot{r}$ only, the $\bo(\epsilon)$ portion of the right hand side of \eqref{eq:pertS-h-eps_first-two-terms} is a function of $t,r,\dot{r}$ alone and not $p$.
		This is required since $h_\epsilon$ is required to be a function of $t,r,\dot{r},\epsilon$ alone, and is the reason that we needed to replace $\ddot{r}$ by $f(t,r,\dot{r},\tilde{h}_0)$ in the $\bo(\epsilon)$ term. Substituting \eqref{eq:pertS-h-eps_first-two-terms} into the first equation of \eqref{eq:pertS-equations-motion-local-w-epsilon} yields Equation \eqref{eq:pertS-corrected-eqn-first-order-f}.
		Finally, making the substitution $f(t,r,\dot{r},\tilde{h}_0) = \ddot{r} + \bo(\epsilon)$ in Equation \eqref{eq:pertS-corrected-eqn-first-order-f} yields Equation \eqref{eq:pertS-corrected-eqn-first-order}.
	\end{proof}

	The following theorem makes clearer the functional form of the dynamics \eqref{eq:pertS-corrected-eqn-first-order}, and it removes the $\bv{g}$ dependence of the right hand side of \eqref{eq:pertS-corrected-eqn-first-order}.
	\begin{thmbis}{th:pertS-dynamics-on-slow-manifold-simpler}
		Assume the hypotheses of Theorem \ref{th:pertS_NAIM_persists}.
		For sufficiently small $\epsilon > 0$, then for each local trivialization there exist smooth fields of linear maps $B(r)$ and $(1,2)$ tensors $G(r)$ such that the dynamics restricted to the slow manifold $M_\epsilon$  in the local trivialization satisfy
		\begin{equation}\label{eq:pertS-dynamics-slow-mfld-solved-for-BV-simpler}
		\bv{g} = -\Av(r) \cdot \dot{r} + \epsilon B(r)\cdot\ddot{r} +\epsilon G(r)\cdot(\dot{r},\dot{r}) + \bo(\epsilon^2).
		\end{equation}
	\end{thmbis}
	\begin{Rem}
		The (1,2) tensors $G(r)$ are \emph{not} generally symmetric, which is clear from Equation \eqref{eq:G-tensors-explicit} below.
	\end{Rem}
	\begin{proof}
		Using the properties of $\ad^*$, we may write $\ad^*_{\bv{g}}(\bar{h}_0) = (C\cdot \bar{h}_0)\cdot (\bv{g})$ for an appropriate ($r$-independent) linear map $C\colon \g^* \to \text{End}(\g)$, and hence we may rewrite \eqref{eq:pertS-corrected-eqn-first-order} as
		\begin{equation*}
		(\id_\g + \epsilon \bVl^{-1} (C\cdot \bar{h}_0) )\cdot (\bv{g}) = -\Av \cdot \dot{r} +\epsilon \bVl^{-1} \left(\left(\frac{\partial}{\partial r} \bar{h}_0\right)\dot{r} + \left(\frac{\partial}{\partial\dot{r}}\bar{h}_0\right)\ddot{r}\right) + \bo(\epsilon^2).
		\end{equation*}
		For sufficiently small $\epsilon$, we may use the identity
		\begin{equation*}
		(\id_\g + \epsilon \bVl^{-1} (C\cdot \bar{h}_0) )^{-1} = \id_\g - \epsilon \bVl^{-1} (C\cdot \bar{h}_0) + \bo(\epsilon^2)
		\end{equation*}
		to obtain
		\begin{equation}\label{eq:pertS-dynamics-slow-mfld-solved-for-BV}
		\bv{g} = -\Av \cdot \dot{r} + \epsilon \bVl^{-1}(C\cdot \bar{h}_0)\cdot \Av\cdot\dot{r} + \epsilon \bVl^{-1} \left(\frac{\partial}{\partial r} \bar{h}_0\right)\dot{r} + \epsilon \bVl^{-1} \left(\frac{\partial}{\partial\dot{r}}\bar{h}_0\right)\ddot{r} + \bo(\epsilon^2).
		\end{equation}
		Since $\bar{h}_0(r,\dot{r})= \bIl(r)(\Am(r)-\Av(r))\cdot\dot{r}$ is linear in $\dot{r}$, it follows that the second and third terms are bilinear in $\dot{r}$, and the fourth term is linear in $\ddot{r}$.
		Hence we may take $B(r)\coloneqq \bVl^{-1} \left(\frac{\partial}{\partial\dot{r}}\bar{h}_0\right)$ and
		\begin{equation}\label{eq:G-tensors-explicit}
		G(r)\cdot(\dot{r},\dot{r})\coloneqq \bVl^{-1}(C\cdot \Il(\Am-\Av)\cdot\dot{r})\cdot \Av\cdot\dot{r} + \epsilon \bVl^{-1} \frac{\partial}{\partial r}\left( \Il(\Am-\Av)\cdot\dot{r}\right)\cdot \dot{r}.
		\end{equation}
	\end{proof}

\bibliographystyle{plainnat}
\bibliography{ref-final}

\end{document}